\documentclass{article}
\usepackage{microtype}
\usepackage{subfigure}

\usepackage[utf8]{inputenc} 
\usepackage[T1]{fontenc}    
\usepackage{hyperref}       
\usepackage{url}            
\usepackage{booktabs}       
\usepackage{amsfonts}       
\usepackage{nicefrac}       
\usepackage{microtype}      
\usepackage{comment}
\usepackage{graphicx}
\usepackage{booktabs} 
\usepackage{nicefrac}
\usepackage{multirow}
\usepackage{mathtools}
\usepackage{relsize}
\usepackage{amsthm}
\usepackage{amssymb}
\usepackage{amsmath}
\usepackage{todonotes}
\usepackage{caption}
\usepackage{times}
\usepackage{float}
\usepackage{bm}
\usepackage{dsfont}
\usepackage{physics}
\usepackage{numprint}
\usepackage{makecell}
\usepackage[linesnumbered,ruled,vlined]{algorithm2e}

\usepackage{stmaryrd}

\usepackage{microtype}
\usepackage{graphicx}
\usepackage{subfigure}
\usepackage{booktabs} 

\usepackage[utf8]{inputenc} 
\usepackage[T1]{fontenc}    
\usepackage{hyperref}       
\usepackage{url}            
\usepackage{booktabs}       
\usepackage{amsfonts}       
\usepackage{booktabs} 
\usepackage{nicefrac}
\usepackage{multirow}
\usepackage{mathtools}
\usepackage{relsize}
\usepackage{amsthm}
\usepackage{amssymb}
\usepackage{amsmath}
\usepackage{todonotes}
\usepackage{times}
\usepackage{bibentry}

\usepackage{float}
\usepackage{bm}
\usepackage{dsfont}
\usepackage{physics}
\usepackage{numprint}

\usepackage{makecell}

\newcommand\numberthis{\addtocounter{equation}{1}\tag{\theequation}}
\allowdisplaybreaks

\usepackage[linesnumbered,ruled,vlined]{algorithm2e}

\newtheorem{definition}{Definition}
\newtheorem{lemma}{Lemma}
\newtheorem{theorem}{Theorem}
\newtheorem*{theorem*}{Theorem}

\DeclareMathOperator{\sign}{sgn}

\DeclareMathOperator*{\argmin}{argmin}
\DeclareMathOperator*{\expect}{ \mathbb{E}}
\DeclareMathOperator*{\score}{\mathcal{R}_{adv}}

\DeclareMathOperator*{\Risk}{\mathcal{R}}

\DeclareMathOperator*{\probset}{\mathcal{P}}
\DeclareMathOperator*{\pen}{\Omega}
\DeclareMathOperator*{\penCW}{\Omega_{\text{norm}}}
\DeclareMathOperator*{\penMass}{\Omega_{\text{mass}}}

\DeclareMathOperator*{\BRD}{BR}
\DeclareMathOperator*{\BRA}{BR_{\pen}}

\DeclareMathOperator*{\scorereg}{\mathcal{R}_{adv}^{\pen}}
\DeclareMathOperator*{\scoreregCW}{\mathcal{R}_{adv}^{\penCW}}
\DeclareMathOperator*{\scoreregMass}{\mathcal{R}_{adv}^{\penMass}}

\DeclareMathOperator*{\essup}{essup}
\DeclareMathOperator*{\essinf}{essinf}

\DeclareMathOperator{\proj}{\pi}

\newtheorem*{addLemma}{Additional Result}

\usepackage{hyperref}

\usepackage[nohyperref, accepted]{icml2020}

\icmltitlerunning{Randomization matters}

\begin{document}
\twocolumn[
\icmltitle{Randomization matters \\
           How to defend against strong adversarial attacks}

\icmlkeywords{Machine Learning, ICML}

\icmlsetsymbol{equal}{*}

\begin{icmlauthorlist}
\icmlauthor{Rafael Pinot}{equal,dauphine,list}
\icmlauthor{Raphael Ettedgui}{equal,dauphine}
\icmlauthor{Geovani Rizk}{dauphine}
\icmlauthor{Yann Chevaleyre}{dauphine}
\icmlauthor{Jamal Atif}{dauphine}
\end{icmlauthorlist}

\icmlaffiliation{dauphine}{Universit\'e Paris-Dauphine, PSL Research University, CNRS, LAMSADE, Paris, France}
\icmlaffiliation{list}{Institut LIST, CEA, Universit\'e Paris-Saclay, France}
\icmlcorrespondingauthor{Rafael Pinot}{rafael.pinot@dauphine.fr}
\icmlcorrespondingauthor{Raphael Ettedgui}{raphael.ettedgui@dauphine.eu}

\icmlkeywords{Machine Learning, ICML}
\vskip 0.3in]



\printAffiliationsAndNotice{\icmlEqualContribution} % otherwise use the standard text.

\begin{abstract}
 \emph{Is there a classifier that ensures optimal robustness against all adversarial attacks?}
This paper tackles this question by adopting a game-theoretic point of view. We present the adversarial attacks and defenses problem as an \emph{infinite} zero-sum game where classical results (\emph{e.g.} Nash or Sion theorems) do not apply. We demonstrate the non-existence of a Nash equilibrium in our game when the classifier and the Adversary are both deterministic, hence giving a negative answer to the above question in the deterministic regime. Nonetheless, the question remains open in the randomized regime. We tackle this problem by showing that any deterministic classifier can be outperformed by a randomized one. This gives arguments for using randomization, and leads us to a simple method for building randomized classifiers that are robust to state-or-the-art adversarial attacks. Empirical results validate our theoretical analysis, and show that our defense method considerably outperforms Adversarial Training against strong adaptive attacks, by achieving 0.55 accuracy under adaptive PGD-attack on CIFAR10, compared to 0.42 for Adversarial training.
\end{abstract}

\section{Introduction}
\label{section::introduction}

Adversarial example attacks recently became a major concern in the machine learning community. An adversarial attack refers to a small, imperceptible change of an input that is maliciously designed to fool a machine learning algorithm. Since the seminal work of~\cite{biggio2013evasion} and~\cite{Szegedy2013IntriguingPO} it became increasingly important to understand the very nature of this phenomenon~\cite{NIPS2016_6331,NIPS2018Fawzi,pmlr-v97-bubeck19a,NIPS2019_8307,NIPS2019_8963}. Furthermore, a large body of work has been published on designing attacks~\cite{goodfellow2014explaining,Papernot2016TheLO,madry2018towards,carlini2017towards, athalye2018obfuscated} and defenses~\cite{goodfellow2014explaining,papernot2016distillation,madry2018towards,KolterRandomizedSmoothing}. 

Besides, in real-life scenarios such as for an autonomous car, errors can be very costly. It is not enough to just defend against new attacks as they are published. We would need an algorithm that behaves optimally against every single attack. However, it remains unknown whether such a defense exists. This leads to the following questions, for which we provide principled and theoretically-grounded answers.

\textbf{Q1:} Is there a deterministic classifier that ensures optimal robustness against any adversarial attack?

\noindent\textbf{A1:} To answer this question, in Section~\ref{section::defninition}, we cast the adversarial examples problem as an \emph{infinite} zero-sum game between a Defender (the classifier) and an Adversary that produces adversarial examples. Then we demonstrate, in Section~\ref{section::deterministicregime}, the non-existence of a Nash equilibrium in the deterministic setting of this game.
This entails that no deterministic classifier can claim to be more robust than all other classifiers against any possible adversarial attack. Another consequence of our analysis is that there is no free lunch for transferable attacks: an attack that works on all classifiers will never be optimal against any of them.

\textbf{Q2:} Would randomized defense strategies be a suitable alternative to defend against strong adversarial attacks?

\noindent\textbf{A2:} We tackle this problem both theoretically and empirically. In Section~\ref{section::mixedregime}, we demonstrate that for any deterministic defense there exists a mixture of classifiers that offers better worst-case theoretical guarantees. Building upon this, we devise a method that generates a robust randomized classifier with a one step boosting method. We evaluate this method, in Section~\ref{section::NumericalExperiments}, against strong adaptive attacks on  CIFAR10 and CIFAR100 datasets. It outperforms Adversarial Training against both $\ell_{\infty}$\textbf{-PGD}~\cite{madry2018towards}, and $\ell_{2}$\textbf{-C\&W}~\cite{carlini2017towards} attacks. More precisely, on CIFAR10, our algorithm achieves $0.55$ (resp. $0.53$) accuracy under attack against these attacks, which is an improvement of $0.13$ (resp. $0.18$) over Adversarial Training.

\section{Related Work}
\label{section::related Work}

Many works have studied adversarial examples, in several different settings. We discuss hereafter the different frameworks that we believe to be related to our work, and discuss the aspects on which our contribution differs from them.

\textbf{Distributionally robust optimization.} The work in~\cite{sinha2018certifiable} addresses the problem of adversarial examples through the lens of distributionally robust optimization. They study a min-max problem where the Adversary manipulates the test distribution while being constrained in a  Wasserstein distance ball (they impose a global constraint on distributions for the Adversary, while we study a local, pointwise constraint, leading to different attack policies). A similar analysis was presented in~\cite{NIPS2018_7534} in a more general setting that does not focus on adversarial examples. Even though our work studies a close problem, our reasoning is very different. We adopt a game theoretic standpoint, which allows us to investigate randomized defenses and endow them with strong theoretical evidences.

\textbf{Game Theory.} 
Some works have tackled the problem of adversarial examples as a two player game. For example~\cite{10.1145/2020408.2020495}  views adversarial example attacks and defenses as a Stackelberg game. More recently, \cite{7533509} and \cite{DBLP:journals/corr/abs-1906-02816} investigated zero-sum games. They consider restricted versions of the game where classical theorems apply, such as when the players only have a finite set of possible strategies. We study a more general setting. Finally, \cite{pruningDefenseICLR2018} motivates the use of noise injection as a defense mechanism by game theoretic arguments but only present empirical results. 

\textbf{Randomization.} Following the work of \cite{pruningDefenseICLR2018} and \cite{Xie2017MitigatingAE}, several recent works studied noise injection as a defense mechanism. In particular, \cite{lecuyer2018certified}, followed by \cite{KolterRandomizedSmoothing,NIPS2019_9143,Pinot2019,NIPS2019_8443} demonstrated that noise injection can, in some cases, give provable defense against adversarial attacks. The analysis and defense method we propose in this paper are not based on noise injection.
However, a link could be made between these works and the mixture we propose, by noting that a classifier in which noise is being injected can be seen as an infinite mixture of perturbed classifiers.

\textbf{Optimal transport.} Our work considers a distributionnal setting, in which the Adversary manipulating the dataset is formalized by a push-forward measure. This kind of setting is close to optimal transport settings recently developed by \cite{NIPS2019_8968} and \cite{pydi2019adversarial}. Specifically, these works investigate classifier-agnostic lower bounds on the risk for binary classification under attack, with some hypothesis on the data distribution. The main differences are that we focus on studying equilibria and not deriving bounds. Moreover, these works do not study the influence of randomization. Finally they express the optimal risk of the Defender in terms of transportation costs between two distributions, whereas we explicitly study the Adversary's behaviour as a transport from one distribution to another. Even though they do not treat the problem from the same prism, we believe that these works are profoundly related and complementary to ours.

\textbf{Ensemble of classifiers.} Some works have been done to improve the robustness of a model by constructing ensemble of classifiers \cite{abbasi2017robustness,xu2017feature,verma2019error,pang2019improving,sen2020empir}. However all the defense methods proposed in those papers subsequently proved to be ineffective against adaptive attacks introduced in~\cite{he2017adversarial,tramer2020adaptive}. The main difference with our method is that it is not an ensemble method since it uses sampling instead of voting to aggregate the classifiers’ output. Hence in terms of volatility, in voting methods, whenever a majority agrees on an opinion, all others votes will be ignored, whereas here each classifier always contributes according to its probability weights, which do not depend on the others.

\section{A Game Theoretic point of view.}

\subsection{Initial problem statement}
\label{section::defninition}

\textbf{Notations.} For any set $\mathcal{Z}$ with $\sigma$-algebra $\sigma\left( \mathcal{Z} \right)$, if there is no ambiguity on the considered $\sigma$-algebra, we denote $\probset\left(\mathcal{Z}\right)$ the set of all probability measures over $\left( \mathcal{Z} , \sigma\left( \mathcal{Z} \right)\right)$, and $\mathcal{F}_{\mathcal{Z}}$ the set of all measurable functions from $\left( \mathcal{Z} , \sigma\left( \mathcal{Z} \right)\right)$ to $\left( \mathcal{Z} , \sigma\left( \mathcal{Z} \right)\right)$. For $\mu \in \probset\left(\mathcal{Z}\right)$ and $\phi \in \mathcal{F}_{\mathcal{Z}}$, the \emph{pushforward measure} of $\mu$ by $\phi$ is the measure $\phi \# \mu$ such that $\phi \# \mu(B) = \mu(\phi^{\text{-}1}(B))$ for any $B \in \sigma(\mathcal{Z}).$

\textbf{Binary classification task.}
Let $\mathcal{X} \subset \mathbb{R}^d$ and $\mathcal{Y}= \{\text{-}1,1\}$. We consider a distribution $\mathcal{D} \in \probset\left(\mathcal{X} \times \mathcal{Y}\right)$ that we assume to be of support $\mathcal{X} \times \mathcal{Y}$. The Defender is looking for a hypothesis (classifier) $h$ in a class of functions $\mathcal{H}$, minimizing the risk of $h$ w.r.t. $\mathcal{D} $:
\begin{equation}
\label{eqn::DefenderRisk}
\begin{split}
    \Risk(h) :&=  \expect_{(X,Y)\sim \mathcal{D}}\left[ \mathds{1} \left\{ h(X) \neq Y \right\}\right] \\ 
    &=  \expect_{Y \sim \nu} \left[  \expect_{X \sim \mu_Y} \left[ \mathds{1} \left\{ h(X) \neq Y \right\}\right] \right].
\end{split}
\end{equation}
Where $\mathcal{H} := \{h:x \mapsto \sign g(x) \mid g:\mathcal{X}\rightarrow \mathbb{R} \textnormal{ continuous}\}$, $\nu \in \probset\left(\mathcal{Y}\right)$ is the probability measure that defines the law of the random variable $Y$, and for any $y \in \mathcal{Y} $, $\mu_y \in \probset\left(\mathcal{X}\right)$ is the conditional law of $X \vert (Y=y)$. \\

\textbf{Adversarial example attack (point-wise).}
Given a classifier $h: \mathcal{X} \rightarrow \mathcal{Y}$ and a data sample $(x,y) \sim \mathcal{D}$, the Adversary seeks a perturbation $\tau \in \mathcal{X}$ that is visually imperceptible, but modifies $x$ enough to change its class, \emph{i.e.} $h(x+\tau) \neq y$. Such a perturbation is called an \textit{adversarial example attack}. In practice, it is hard to evaluate the set of visually imperceptible modifications of an image.
However, a sufficient condition to ensure that the attack is undetectable is to constrain the perturbation $\tau$ to have a small norm, be it for the $\ell_{\infty}$ or the $\ell_2$ norm. Hence, one should always ensure that $\norm{\tau}_{\infty} \leq \epsilon_\infty$, or $\norm{\tau}_2 \leq \epsilon_ 2$, depending on the norm used to measure visual imperceptibility. The choice of the threshold depends on the application at hand. For example, on CIFAR datasets, typical values for $\epsilon_\infty$ and $\epsilon_2$ are respectively, $0.031$ and $0.4/0.6/0.8$. In the remaining of this work, we will define our constraint using an $\ell_2$ norm, but all our results are valid for an $\ell_{\infty}$ based constraint. 

\textbf{Adversarial example attack (distributional).} The Adversary chooses, for every $x \in \mathcal{X}$, a perturbation that depends on its true label $y$. This amounts to construct, for each label $y\in \mathcal{Y}$, a measurable function $\phi_y$ such that $\phi_y(x)$ is the perturbation
associated with the labeled example $(x,y)$. This function naturally induces a probability distribution over adversarial examples, which is simply the push-forward measure $\phi_y \# \mu_y$.
The goal of the Adversary is thus to find  $\boldsymbol{\phi} = (\phi_{\text{-}1},\phi_{1}) \in (\mathcal{F}_{\mathcal{X} \vert \epsilon_2})^2$ that maximizes the adversarial risk $\score(h,\boldsymbol{\phi})$ defined as follows:
\begin{equation}
\label{eqn::AttackerRisk}
\begin{split}
     \score(h,\boldsymbol{\phi}):= \expect_{Y \sim \nu} \left[  \expect_{X \sim \phi_{\small{Y}} \# \mu_Y} \left[ \mathds{1} \left\{ h(X) \neq Y \right\}\right] \right].
\end{split}
\end{equation}
Where for any $\epsilon_2 \in (0,1)$, $\mathcal{F}_{\mathcal{X} \vert \epsilon_2}$ is the set of functions that imperceptibly modifies a distribution: 
\begin{equation*}
\begin{split}
     \mathcal{F}_{\mathcal{X} \vert \epsilon_2}:= \left \{ \psi \in \mathcal{F}_{\mathcal{X}} \mid \essup\limits_{x \in \mathcal{X}}\norm{\psi(x) - x}_2 \leq \epsilon_2 \right \}.
\end{split}
\end{equation*}

%The penalty function $\pen$ represents the transportation cost of the Adversary, and $\lambda \in (0,1)$ some regularization weight. More precisely, $\pen$ encodes the constraints that the Adversary enforces on the attacks, to remain undetected.

%For now, we only impose the constraint that every attack should be bounded by $\epsilon$, i.e.

%\begin{equation}
%\begin{split}
%    \label{penalty::basic}
%    \Omega(\phi) := \infty \mathds{1}\left\{ \norm{X - \phi_Y(X)}_2 > \epsilon_2 \right\} \big] \Big],
%\end{split}
%\end{equation} 

%Whatever penalty we choose, its value does not depend on $h$ (hence, the optimal Defender minimizing Eq.~\eqref{eqn::DefenderRisk} or~\eqref{eqn::AttackerRisk} is the same).

\textbf{Adversarial defense, a two-player zero-sum game.} With the setting defined above, the adversarial examples problem can be seen as a two-player zero-sum game, where the Defender tries to find the best possible hypothesis $h$, while a strong Adversary is manipulating the dataset distribution:
\begin{equation}
\label{InfSup}
\inf\limits_{h \in \mathcal{H}} \sup\limits_{ \boldsymbol{\phi} \in \left(\mathcal{F}_{\mathcal{X} \vert \epsilon_{2}}\right)^2} \score(h, \boldsymbol{\phi}). 
\end{equation}
This means that the Defender tries to design the classifier with the best performance under attack, whereas the Adversary will each time design the optimal attack on this specific classifier. 
In the game theoretical terminology, the choice of a classifier $h$ (resp. an attack $\boldsymbol{\phi}$) for the Defender (resp. the Adversary) is called a \emph{strategy}. It is crucial to note that the $\sup$-$\inf$ and $\inf$-$\sup$ problems do not necessarily coincide. In this paper, we mainly focus on the Defender's point of view which corresponds to the $\inf$-$\sup$ problem. We will be interested in understanding the behaviour of players in this game, \emph{i.e.} the best responses they have to a given strategy, and whether some equilibria may arise. This motivates the following definitions.

\begin{definition}[Best Response]
Let $h \in \mathcal{H}$, and $\boldsymbol{\phi} \in \left(\mathcal{F}_{\mathcal{X} \vert \epsilon_{2}}\right)^2$. A \emph{best response} from the Defender to $\boldsymbol{\phi}$ is a classifier $h^{*} \in \mathcal{H}$ such that $\score(h^{*},\boldsymbol{\phi}) = \min\limits_{h \in \mathcal{H}} \score(h,\boldsymbol{\phi})$.
Similarly, a \emph{best response} from the Adversary to $h$ is an attack $\boldsymbol{\phi}^{*} \in \left(\mathcal{F}_{\mathcal{X} \vert \epsilon_{2}}\right)^2$ such that $\score(h, \boldsymbol{\phi}^{*}) = \max\limits_{\boldsymbol{\phi} \in \left(\mathcal{F}_{\mathcal{X} \vert \epsilon_{2}}\right)^2} \score(h, \boldsymbol{\phi})$.
\label{def:BestResponse}
\end{definition}

In the remaining, we denote $\BRD(h)$ the set of all best responses of the Adversary to a classifier $h$. Similarly $\BRD(\boldsymbol{\phi})$ denotes the set of best responses to an attack $\boldsymbol{\phi}$.

\begin{definition}[Pure Nash Equilibrium]
\label{def:NashEq}
In the zero-sum game (Eq.~\ref{InfSup}), a \emph{Pure Nash Equilibrium} is a couple of strategies $(h, \boldsymbol{\phi}) \in \mathcal{H} \times \left(\mathcal{F}_{\mathcal{X} \vert \epsilon_{2}}\right)^2 $ such that
$$ \left\{
    \begin{array}{ll}
    h &\in \BRD(\boldsymbol{\phi}), \text{ and,} \\
    \boldsymbol{\phi} &\in \BRD(h).
    \end{array}
    \right. $$
\end{definition}

When it exists, a Pure Nash Equilibrium is a state of the game in which no player has any incentive to modify its strategy. In our setting, this simultaneously means that no attack could better fool the current classifier, and that the classifier is optimal for the current attack.

\textbf{Remark.} All the definitions in this section assume a deterministic regime, \emph{i.e.} that neither the Defender nor the Adversary use randomization, hence the notion of \emph{Pure} Nash Equilibrium in the game theory terminology. The randomized regime will be studied in Section~\ref{section::mixedregime}.

\subsection{Trivial solution and Regularized Adversary}

\textbf{Trivial Nash equilibrium.} Our current definition of the problem implies that the Adversary has perfect information on the dataset distribution and the classifier. It also has unlimited computational power and no constraint on the attack except on the size of the perturbation. Going back to the example of the autonomous car, this would mean that the Adversary can modify every single image that the camera \emph{may} receive during \emph{any} trip, which is highly unrealistic. The Adversary has no downside to attacking, even when the attack is unnecessary, \emph{e.g.} if the attack cannot work or if the point is already misclassified. 

This type of behavior for the Adversary can lead to the existence of a pathological (and trivial) Nash Equilibrium as demonstrated in Figure~\ref{fig:attaqueGaussiennes} for the uni-dimensional setting with Gaussian distributions. The unbounded Adversary moves every point toward the decision boundary (each time maximizing the perturbation budget), and the Defender cannot do anything to mitigate the damage. In this case the decision boundary for the Optimal Bayes Classifier remains unchanged, even though both curves have been moved toward the center, hence a trivial equilibrium.
In the remaining of this work, we show that such an equilibrium does not exist as soon as there is a small restraint on the Adversary's strength, \emph{i.e.} as soon as it is not perfectly indifferent to produce unnecessary perturbations.

\begin{figure*}[ht]
\centering
    \includegraphics[width=\textwidth]{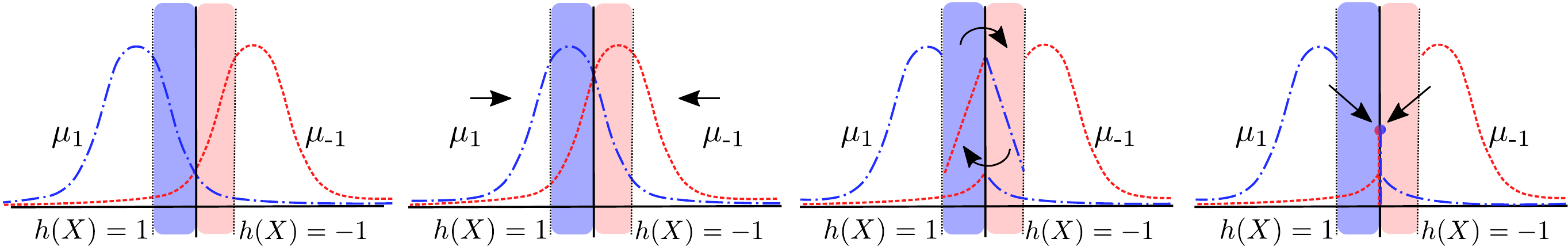}
    \caption{Representation of the $\mu_{\text{-}1}$ (blue dotted line) and $\mu_1$ (red plain line) distributions, without attack (left) and with three different attacks: no penalty (second drawing), with mass penalty (third) and with norm penalty (fourth). On all figures blue area on the left of the axis is $P_h(\epsilon_2)$ and red area on the right is $N_h(\epsilon_2)$.}
    \label{fig:attaqueGaussiennes}
\end{figure*}

\textbf{Regularized Adversary.} 
To mitigate the Adversary strength, we introduce a penalization term: 

\begin{equation}
\label{InfSupreg}
     \inf\limits_{h \in \mathcal{H}} \sup\limits_{ \boldsymbol{\phi} \in \left(\mathcal{F}_{\mathcal{X} \vert \epsilon_{2}}\right)^2}  \underbrace{\left[ \score(h, \boldsymbol{\phi}) 
      - \lambda \pen\left(\boldsymbol{\phi}\right) \right]}_{\textstyle \scorereg(h, \boldsymbol{\phi})}.
\end{equation}

The penalty function $\pen$ represents the limitations on the Adversary's budget, be it because of computational resources or to avoid being detected. $\lambda \in (0,1)$ is some regularization weight. In this paper, we study two types of penalties: the \emph{mass penalty} $\penMass$, and the \emph{norm penalty} $\penCW$. 

From a computer-security point of view, the first limitation that comes to mind is to limit the number of queries the Adversary can send to the classifier. In our distributional setting, this boils down to penalizing the mass of points that the function $\boldsymbol{\phi}$ moves. Hence we define the mass penalty as:  
\begin{equation}
    \label{penalty::massPenalty}
    \penMass(\boldsymbol{\phi}) :=  \expect_{Y \sim \nu} \left[  \expect_{X \sim \mu_Y} \left[ \mathds{1}\left\{ X \neq \phi_Y(X)\right\} \right] \right].
\end{equation}
The mass penalty discourages the Adversary from attacking too many points by penalizing the overall mass of transported points. The second limitation we consider penalizes the expected norm under $\boldsymbol{\phi}$:  
\begin{equation}
\begin{split}
    \label{penalty::CW}
    \penCW(\boldsymbol{\phi}) :=  \expect_{Y \sim \nu} \left[ \expect_{X \sim \mu_Y} \left[ \norm{X -\phi_Y(X)}_2  \right] \right].
\end{split}
\end{equation} 
This regularization is very common in both the optimization and adversarial example communities. In particular, it is used by Carlini $\&$ Wagner~\cite{carlini2017towards} to compute the eponymous attack\footnote{$\penCW$ is not limited to $\ell_2$ norm. The results we present hold as long as the norm used to compare $X$ and $\phi_Y(X)$ comes from a scalar product on $\mathcal{X}$.}.
In the following, we denote $\BRD_{\penMass}$ (resp. $\BRD_{\penCW}$) the best responses for the Adversary w.r.t the mass (resp. norm) penalty. Section~\ref{section::deterministicregime} shows that whatever penalty the Adversary has, no Pure Nash Equilibrium exists. We characterize the best responses for each player, and show that they can never satisfy Definition~\ref{def:NashEq}.

\section{Deterministic regime}
\label{section::deterministicregime}

\textbf{Notations.} Let $h \in \mathcal{H}$, we denote $P_h := \left\{  x\in \mathcal{X} \mid  h(x) = 1 \right\}$, and $N_h := \left\{ x \in \mathcal{X} \mid h(x) = \text{-}1 \right\}$ respectively the set of positive and negative outputs of $h$. We also  denote the set of attackable points from the positive outputs $P_h(\delta) := \left\{ x \in P_h \mid \exists z \in N_h \text{ and } \Vert z-x \Vert_2 \leq \delta \right\}$, and $N_h(\delta)$ likewise.
    
\textbf{Adversary's best response.} Let us first present the best responses of the Adversary under respectively the mass penalty and the norm penalty. Both best responses share a fundamental behavior: the optimal attack will only change points that are close enough to the decision boundary. This means that, when the Adversary has no chance of making the classifier change its decision about a given point, it will not attack it. However, for the norm penalty all attacked points are projected on the decision boundary, whereas with the mass penalty the attack moves the points across the border.

\begin{lemma}
\label{lemma:BRattMass}
Let $h \in \mathcal{H}$ and $\boldsymbol{\phi} \in \BRD_{\penMass}(h)$. Then the following assertion holds: $$
    \left\{
    \begin{array}{ll}
    \phi_1(x) \in (P_h)^{\complement} &\text{ if }  x \in P_h(\epsilon_2) \\
    \phi_1(x) = x &\text{ otherwise}.
    \end{array}
    \right.
$$ 
Where $(P_h)^{\complement}$, the complement of $P_h$ in $\mathcal{X}$. $\phi_{\text{-}1}$ is characterized symmetrically.
\end{lemma}

\begin{lemma}
\label{lemma:BRattCW}
Let $h \in \mathcal{H}$ and  $\boldsymbol{\phi} \in \BRD_{\penCW}(h)$. Then the following assertion holds: $$
    \phi_1 (x) = \left\{
    \begin{array}{ll}
    \pi(x) &\text{ if } x \in P_h(\epsilon_2)\\
    x &\text{ otherwise}.
    \end{array}
    \right.
$$
Where $\pi$ is the orthogonal projection on $(P_h)^{\complement}$. $\phi_{\text{-}1}$ is characterized symmetrically. 
\end{lemma}

%\begin{figure}
%\centering
%    \includegraphics[width=\columnwidth]{Images/Ressource-Bounded-GT-Comparison-Mass-Norm-Penalty.pdf}
%    \caption{Representation of the $\mu_{\text{-}1}$ (blue dotted line) and $\mu_1$ (red plain line) distributions, without attack (left) and with three different attacks: no penalty (second drawing), with mass penalty (third) and with norm penalty (fourth). On all figures blue area on the left of the axis is $P_h(\epsilon_2)$ and red area on the right is $N_h(\epsilon_2)$.}
%    \label{fig:attaqueGaussiennespenalties}
%\end{figure}

 These best responses are illustrated in Figure~\ref{fig:attaqueGaussiennes} with two uni-dimensional Gaussian distributions. For the mass penalty, $\mu_1$ is set to $0$ in $P_h(\epsilon_2)$, and this mass is transported into $N_h(\epsilon_2)$. The symmetric holds for $\mu_{\text{-}1}$. After attack, we now have $\mu_1\left(P_h(\epsilon_2)\right)=0$, so a small value of $\mu_{\text{-}1}$ in $P_h(\epsilon_2)$ suffices to make it dominant, and that zone will now be classified -$1$ by the Optimal Bayes Classifier.  For the norm penalty, the part of $\mu_1$ that was in $P_h(\epsilon_2)$ is transported on a Dirac distribution at the decision boundary. Similarly to the mass penalty, the best response now predicts $\text{-}1$ for the zone $P_h(\epsilon_2)$.

\textbf{Remark.} In practice, it might be computationally hard to generate the exact best response for the norm penalty, \emph{i.e.} the projection on the decision boundary. That will happen for example if this boundary is very complex (\emph{e.g.} highly non-smooth), or when $\mathcal{X}$ is in a high dimensional space. To keep the attack tractable, the Adversary will have to compute an approximated best response by allowing the projection to reach the point within a small ball around the boundary. This means that the best responses of the norm penalty and the mass penalty problems will often match.

\textbf{Defender's best response.} At a first glance, one would suspect that the best response for the Defender ought to be the Optimal Bayes Classifier for the transported distribution. However, it is only well defined if the conditional distributions admit a probability density function. This might not always hold here for the transported distribution. Nevertheless, we show that there is a property, shared by the Optimal Bayes Classifier when defined, that always holds for the Defender's best response.

\begin{lemma}
\label{lemma:BRDefender}
Let us consider $\boldsymbol{\phi} \in \left(\mathcal{F}_{\mathcal{X} \vert \epsilon_{2}}\right)^2$. If we take $h \in \BRD(\boldsymbol{\phi})$, then for $y =1$ (resp. $y=\text{-}1$), and for any $B \subset P_h$ (resp. $B \subset N_h$) one has $$ \mathbb{P}(Y=y \vert X \in B) \geq \mathbb{P}(Y= -y \vert X \in B) $$ with $Y \sim \nu$ and for all $y \in \mathcal{Y}$, $X|(Y=y) \sim \phi_y \# \mu_y $.
\end{lemma}

In particular, when $\phi_1 \# \mu_1$ and $\phi_{\text{-}1} \# \mu_{\text{-}1}$ admit probability density functions, Lemma~\ref{lemma:BRDefender} simply means that $h$ is the Optimal Bayes Classifier for the distribution $(\nu,\phi_1 \# \mu_1, \phi_{\text{-}1} \# \mu_{\text{-}1})$\footnote{We prove this result in the supplementary material.}. We can now state our main theorem, as well as two of its important consequences.

\begin{theorem}[Non-existence of a pure Nash equilibrium]
\label{theorem::PureNashEquilibrium}
In the zero-sum game (Eq.~\ref{InfSupreg}) with $\lambda \in (0,1)$ and penalty $\Omega \in \{ \penMass,\penCW \}$, there is no Pure Nash Equilibrium.
\end{theorem}

\textbf{Consequence 1.} (\emph{No free lunch for transferable attacks})
To understand this statement, remark that, thanks to weak duality, the following inequality always holds: 

\begin{equation*}
\begin{split}
    \sup\limits_{\boldsymbol{\phi} \in \left(\mathcal{F}_{\mathcal{X} \vert \epsilon_{2}}\right)^2} \inf\limits_{h \in \mathcal{H}} \scorereg(h, \boldsymbol{\phi})
    \leq \inf\limits_{h \in \mathcal{H}} \sup\limits_{\boldsymbol{\phi} \in \left(\mathcal{F}_{\mathcal{X} \vert \epsilon_{2}}\right)^2} \scorereg(h, \boldsymbol{\phi}).
\end{split}
\end{equation*}

On the left side problem ($\sup$-$\inf$), the Adversary looks for the best strategy $\boldsymbol{\phi}$ against any \emph{unknown} classifier. This is tightly related to the notion of \emph{transferable attacks} (see \emph{e.g.} ~\cite{tramer-papernot2017transferable}), which refers to attacks successful against a wide range of classifiers. On the right side (our) problem ($\inf$-$\sup$), the Defender tries to find the best classifier under any possible attack, whereas the Adversary plays in second and specifically attacks this classifier.
As a consequence of Theorem~\ref{theorem::PureNashEquilibrium}, the inequality is always strict:
\begin{equation*}
    \sup\limits_{\boldsymbol{\phi} \in \left(\mathcal{F}_{\mathcal{X} \vert \epsilon_{2}}\right)^2} \inf\limits_{h \in \mathcal{H}} \scorereg(h, \boldsymbol{\phi}) \boldsymbol{<} \inf\limits_{h \in \mathcal{H}} \sup\limits_{\boldsymbol{\phi} \in \left(\mathcal{F}_{\mathcal{X} \vert \epsilon_{2}}\right)^2} \scorereg(h, \boldsymbol{\phi}).
\end{equation*}

This means that both problems are not equivalent. 
In particular, an attack designed to succeed against \emph{any} classifier (\emph{i.e.} a transferable attack) will not be as good as an attack tailored for a given classifier. Hence she has to trade-off between effectiveness and transferability of the attack.

\textbf{Consequence 2.} (\emph{No deterministic defense may be proof against every attack})
Let us consider the state-of-the-art defense which is Adversarial Training \cite{goodfellow2014explaining,madry2018towards}. The idea is to compute an efficient attack $\boldsymbol{\phi}$, and train the classifier on created adversarial examples, in order to move the decision boundary and make the classifier more robust to new perturbations by $\boldsymbol{\phi}$. 

To be fully efficient, this method requires that $\boldsymbol{\phi}$ remains an optimal attack on $h$ even after training. Our theorem shows that it is never the case: after training our classifier $h$ to become ($h'$) robust against $\boldsymbol{\phi}$, there will always be a different optimal attack $\boldsymbol{\phi}'$ that is efficient against $h'$. Hence Adversarial Training will never achieve a perfect defense.

\section{Randomization matters}
\label{section::mixedregime}

As we showed that there is no Pure Nash Equilibrium, no deterministic classifier may be proof against every attack. 
We would therefore need to allow for a wider class of strategies. A natural extension of the game would thus be to allow randomization for both players, who would now choose a distribution over pure strategies, leading to this game:
\begin{equation}
\label{InfSupMixed}
\inf\limits_{\eta \in \probset\left(\mathcal{H}\right) } \sup\limits_{\varphi \in \probset\left(\left(\mathcal{F}_{\mathcal{X} \vert \epsilon_{2}}\right)^2\right)} \expect_{\substack{h \sim \eta \\ \boldsymbol{\phi} \sim \varphi}}\left[ \scorereg(h, \boldsymbol{\phi})  \right].
\end{equation}
Without making further assumptions on this game (e.g. compactness),
we cannot apply known results from game theory (e.g. Sion theorem) to prove the existence of an equilibrium. These assumptions would however make the problem loose much generality, and do not hold here. 

\textbf{Randomization matters.} Even without knowing if an equilibrium exists in the randomized setting, we can prove that \emph{randomization matters}. More precisely we show that any deterministic classifier can be outperformed by a randomized one in terms of the worst case adversarial risk. To do so we simplify Equation~\ref{InfSupMixed} in two ways: 
\begin{enumerate}
    \item We do not consider the Adversary to be randomized, \emph{i.e.} we restrict the search space of the Adversary to $(\mathcal{F}_{\mathcal{X}})^2$ instead of $\probset\left((\mathcal{F}_{\mathcal{X}})^2\right)$. This condition corresponds to the current state-of-the-art in the domain: to the best of our knowledge, no efficient randomized adversarial example attack has been designed (and so is used) yet. 
    \item We only consider a subclass of randomized classifiers, called mixtures, which are discrete probability measures on a finite set of classifiers. We show that this kind of randomization is enough to strictly outperform any deterministic classifier. We will discuss later the use of more general randomization (such as noise injection) for the Defender. Let us now define a mixture of classifiers.
\end{enumerate}
\begin{figure*}[t]
    \centering
    \includegraphics[width=\textwidth]{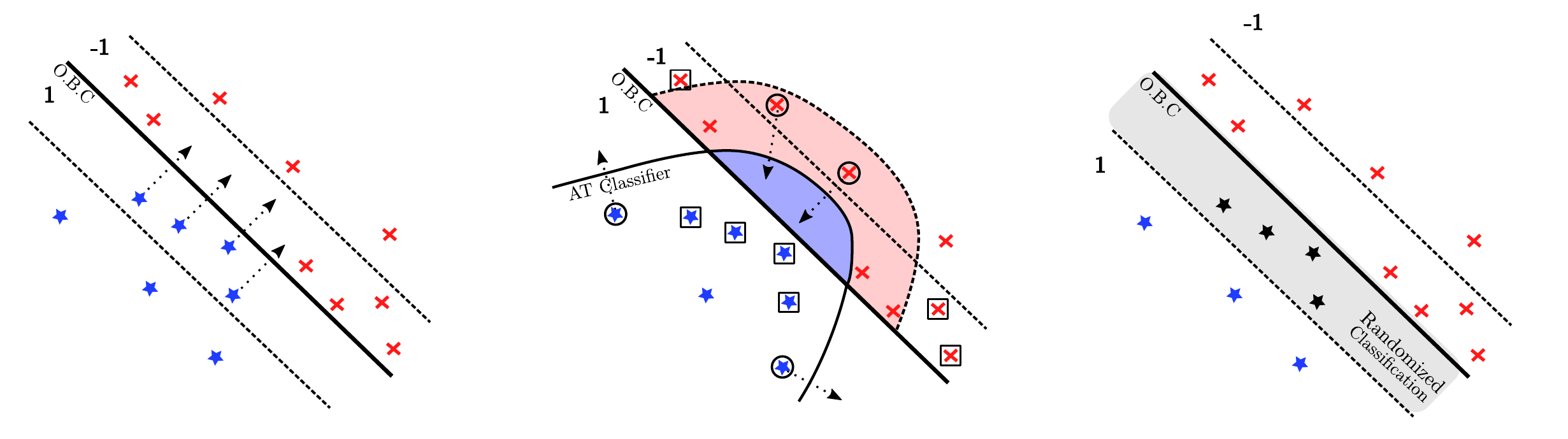}
    \caption{Illustration of adversarial examples (only on class $1$ for more readability) crossing the decision boundary (left), adversarially trained classifier for the class $1$ (middle), and a randomized classifier that defends class $1$. Stars are natural examples for class $1$, and crosses are natural examples for class $\text{-}1$. The straight line is the optimal Bayes classifier, and dashed lines delimit the points close enough to the boundary to be attacked resp. for class $1$ and -$1$. We focus the drawing on the star points. Crosses can be treated symmetrically.}  
    \label{fig::ATvsMixture}
\end{figure*}

\begin{definition}[Mixture of classifier]
\label{def:mixedStrat}
Let $n \in \mathbb{N}$, $\mathbf{h} = (h_1,...,h_n) \in \mathcal{H}^n$ , and $\mathbf{q} \in \probset\left(\{1,...,n\}\right)$. A \emph{mixed classifier of $\mathbf{h}$ by $\mathbf{q}$} is a mapping $m_{\mathbf{h}}^{\mathbf{q}}$ from $\mathcal{X}$ to $\probset\left(\mathcal{Y}\right)$ such that for all $x \in \mathcal{X}$, $m_{\mathbf{h}}^{\mathbf{q}}(x)$ is the discrete probability distribution that is defined for all $y \in \mathcal{Y}$ as follows:
\begin{equation*}
    m_{\mathbf{h}}^{\mathbf{q}}(x)(y) := \expect_{i \sim \mathbf{q}}\left[ \mathds{1}\left\{h_i(x)=y \right\} \right].
\end{equation*}
\end{definition}
We call such a mixture a \emph{mixed strategy} of the Defender. Given some $x \in \mathcal{X}$, this amounts to picking a classifier $h_i$ from $\mathbf{h}$ at random following the distribution $\mathbf{q}$, and use it to output the predicted class for $x$, \emph{i.e.} $h_i(x)$.
Note that a mixed strategy for the Defender is a non deterministic algorithm, since it depends on the sampling one makes on $\mathbf{q}$. Hence, even if the attacks are defined in the same way as before, the Adversary now needs to maximize a new objective function which is the expectation of the adversarial risk under the distribution $m_{\mathbf{h}}^{\mathbf{q}}$. It writes as follows:

\begin{equation}
\begin{split}
\label{eqn::MixedRisk}
     \expect_{Y \sim \nu} \left[  \expect_{X \sim \phi_{\small{Y}} \# \mu_Y} \left[ \expect_{\hat{Y} \sim m_{\mathbf{h}}^{\mathbf{q}}(X)} \left[ \mathds{1} \left\{ \hat{Y} \neq Y \right\} \right] \right] \right] - \lambda \pen\left(\boldsymbol{\phi}\right).
\end{split}
\end{equation}
We also write $\scorereg$ to mean the left part of Equation~(\ref{eqn::MixedRisk}), when it is clear from context that the Defender uses a mixed classifier. Using this new set of strategies for the Defender, we can study whether mixed classifiers outperform deterministic ones, and how to efficiently design them.

\textbf{Mixed strategy.} We demonstrate that the efficiency of any deterministic defense can be improved using a simple mixed strategy. This method presents similarities with the notions of fictitious play~\cite{brown1951iterative} in game theory, and boosting in machine learning~\cite{freund95decisiontheoretic}.
Given a deterministic classifier $h_1$, we combine it (via randomization) with the best response $h_2$ to its optimal attack. 

The rational behind this idea is that, by construction, efficient attacks on one of these two classifiers will not work on the other. Mixing $h_1$ with $h_2$ has two opposite consequences on the adversarial risk. On one hand, where we only had to defend against attack on $h_1$, we are now also vulnerable to attacks on $h_2$, so the total set of possible attacks is now bigger. On the other hand, each attack will only work part of the time, depending on the probability distribution $\textbf{q}$. If we can calibrate the weights so that attacks on important zones have a low probability of succeeding, then the average risk under attack on the mixture will be low. 

\textbf{Toy example where a mixture outperforms AT.} To better understand how randomization can work, let us look at a simple toy example. Figure~\ref{fig::ATvsMixture} illustrates a binary classification setting between two set of points. Attacking the Optimal Bayes Classifier (bold straight line) consists in moving all the points that lie between the dotted lines to the opposite side of the decision boundary (Figure~\ref{fig::ATvsMixture}, left). The general tactic to defend against an attack is to change the classifier's output for points that are too close to the boundary. This can be done all the time, as in Adversarial Training (where we move the decision boundary to incorporate adversarial examples), or part of the time as in a randomized algorithm (so that the attack only works with a given probability). 

When we use Adversarial Training for the star points (Figure~\ref{fig::ATvsMixture}, middle), we change the output on the blue zone, so that 2 of the star (squared) points cannot be successfully attacked anymore. But in exchange, the dilation of the new boundary can now be attacked. For Adversarial Training to work, we need the number of new potential attacks (\emph{i.e.} the points that are circled, $2$ crosses in the dilation and $2$ stars that are close to the new boundary) to be smaller than the number of attacks we prevent (the squared points, $2$ blue ones that an attack would send in the blue zone, and $3$ red points that are far from the new decision boundary). Here we prevent $5$ attacks at the cost of $4$ new ones, so the Adversarial Training improves the total score from $8$ to $7$.

Similarly, we observe what happens for the randomized defense (Figure~\ref{fig::ATvsMixture}, right). We mix the Optimal Bayes Classifier with the best response to attacking all the points. We get a classifier that is determinsitic outside the gray area, and random inside it\footnote{The grey area should actually  be bigger since the best response to the attack would also change the decision on the upper part between the OBC and the doted line. We focus on what happens on the star points for simplicity.}. If the first classifier has a weight $\alpha=0.5$, $6$ of the old attacks now succeed only with probability $0.5$ (crosses between the dotted lines), whereas 3 new attacks are created (stars outside of the gray area) that succeed with probability $0.5$ also. At the end, the average rate of successful attacks is $6.5$, where adversarial training previously achieved $7$.   

More formally, Theorem~\ref{theorem:MixedRegime} shows that whatever penalty we consider, a deterministic classifier can always be outperformed by a randomized algorithm. We now can state our second main result: randomization matters.

\begin{theorem}(Randomization matters)
\label{theorem:MixedRegime}
Let us consider $h_1 \in \mathcal{H}$, $\lambda \in (0,1)$, $\pen = \penMass$, $\boldsymbol{\phi} \in \BRD_{\pen}(h_1)$ and $h_2 \in \BRD(\boldsymbol{\phi})$. Then for any $\alpha \in (\max(\lambda, 1-\lambda),1)$ and for any $\boldsymbol{\phi}' \in \BRD_{\pen}(m^{\mathbf{q}}_{\mathbf{h}})$ one has
\begin{equation*}
     \scorereg(m^{\mathbf{q}}_{\mathbf{h}}, \boldsymbol{\phi}') < \scorereg(h_1, \boldsymbol{\phi}).
\end{equation*}
Where $\mathbf{h}=(h_1,h_2)$, $\mathbf{q}=(\alpha, 1-\alpha)$, and $m^{\mathbf{q}}_{\mathbf{h}}$ is the mixture of $\mathbf{h}$ by $\mathbf{q}$. A similar result holds when $\pen = \penCW$ (see supplementary materials).
\end{theorem}

\textbf{Remark}
Note that depending on the initial hypothesis $h_1$ and the conditional distributions $\mu_1$ and $\mu_{\text{-} 1}$, the gap between $\scorereg(m^{\boldsymbol{q}}_{\boldsymbol{h}},\boldsymbol{\phi}') $and $ \scorereg(h_1, \boldsymbol{\phi})$ could vary. Hence, with additional conditions on  $h_1$, $\mu_1$ and $\mu_{\text{-}1}$, we could make the gap appear more explicitly. We keep the formulation general to emphasize that for \emph{any} deterministic classifier, there exists a randomized one that outperforms it in terms of worst-case adversarial score.

Based on Theorem~\ref{theorem:MixedRegime} we devise a new procedure called Boosted Adversarial Training (BAT) to construct a robust mixture of two classifiers. It is based on three core principles: Adversarial Training, Boosting and Randomization.

\section{Experiments: How to build the mixture}
\label{section::NumericalExperiments}

\textbf{Simple mixture procedure (BAT).}  Given a dataset $D$ and a weight parameter $\alpha \in [0,1]$,
we construct $h_1$ the first classifier of the mixture using Adversarial Training\footnote{We use  $\ell_\infty$\textbf{-PGD} with $20$ iterations and $\epsilon_\infty = 0.031$ to train the first classifier and to build $\tilde{D}$.} on $D$. Then, we train the second classifier $h_2$ on a data set $\tilde{D}$ that contains adversarial examples against $h_1$ created from examples of $D$. At the end we return the mixture constructed with those two classifiers where the first one has a weight of $1-\alpha$ and the second one a weight of $\alpha$. The parameter $\alpha$ is found by conducting a grid-search. In Table~\ref{tab_accuracy_under_attack} we present results for $\alpha = 0.2$ under strong state-of-the-art attacks. The procedure is summarized in Algorithm~\ref{algorithm:Boosting}\footnote{More algorithmic and implementation details can be found in the supplementary materials.} 

\begin{algorithm}[H]
    \SetAlgoLined
    \textbf{Input} : $D$ the training data set and $\alpha$ the weight parameter.
    \vspace{0.3cm} \\
    Create and adversarially train $h_1$ on $D$ \\
    Generate the adversarial data set $\tilde{D}$ against $h_1$. \\
    Create and naturally train $h_2$ on $\tilde{D}$ \\
    ${\bf q} \leftarrow \left(1-\alpha, \alpha \right)$ \\
    ${\bf h} \leftarrow (h_1, h_2)$
    \vspace{0.3cm} \\
    return $m^{\mathbf{q}}_{\mathbf{h}}$
    \caption{Boosted Adversarial Training}
    \label{algorithm:Boosting}
\end{algorithm}

\begin{table*}[t]
\begin{center}
\begin{tabular}{ c | c | c | c | c c c }

 \multirow{2}{*}{\textbf{Dataset}} & \multirow{2}{*}{\textbf{Method}}& \textbf{Natural} & \textbf{Adaptive-}$l_{\infty}$\textbf{-PGD} & \multicolumn{3}{c}{\textbf{Adaptive-}$\ell_{2}$\textbf{-C\&W}}   \\ 
  &  & \textbf{Accuracy} & $\epsilon_\infty = 0.031$ & $\epsilon_2 = 0.4$ & $\epsilon_2 = 0.6$ & $\epsilon_2 = 0.8$ \\ 
 \Xhline{2\arrayrulewidth}
 \multirow{3}{4em}{CIFAR10} & Natural & 0.88 & 0.00 & 0.00 & 0.00 & 0.00\\ 
 &AT~\cite{madry2018towards} & 0.83 & 0.42 & \textbf{0.60} &0.47 &0.35\\ 
 &Ours & 0.80 & \textbf{0.55} &\textbf{0.60} &\textbf{0.57} &\textbf{0.53} \\ 

 \Xhline{2\arrayrulewidth}
 \multirow{3}{4em}{CIFAR100}&Natural & 0.62 & 0.00 & 0.00 & 0.00 & 0.00\\
 &AT~\cite{madry2018towards} & 0.58 & 0.26 & 0.38 & 0.29 & 0.22\\ 
 &Ours & 0.56 & \textbf{0.40} &\textbf{0.45} &\textbf{0.41} &\textbf{0.38} \\ 
\end{tabular}
\end{center}
\caption{Evaluation on CIFAR10 and CIFAR100 without \emph{data augmentation}. Accuracy under attack of a single adversarially trained classifier (AT) and the mixture formed with our method (Ours). The evaluation is made with \textbf{Adaptive-}$\ell_{\infty}$\textbf{-PGD} and \textbf{Adaptive-}$\ell_{2}$\textbf{-C\&W} attacks both computed with 100 iterations. For \textbf{Adaptive-}$\ell_{\infty}$\textbf{-PGD} we use an epsilon equal to $8/255$ ($\approx 0.031$), a step size equal to $2/255$ ($\approx 0.008$) and we allow random initialization.  For \textbf{Adaptive-}$\ell_{2}$\textbf{-C\&W} we use a learning rate equal to 0.01, 9 binary search steps, the initial constant to 0.001, we allow the abortion when it has already converged and we give the results for the different values of rejection threshold $\epsilon_2 \in \{0.4,0.6,0.8\}$. As for EOT, we don't need to estimate the expected accuracy of the mixture through Monte Carlo sampling since we have the exact weight of each classifier of the mixture. Thus we give the exact expected accuracy.}
\label{tab_accuracy_under_attack}
\end{table*}

\textbf{Comparison to fictitious play.}
Contrary to classical algorithms such as \emph{Fictitious play} that also generates mixtures of classifiers, and whose theoretical guarantees rely on the existence of a Mixed Nash Equilibrium, the performance of our method is ensured by Theorem~\ref{theorem:MixedRegime} to be at least as good as the classifier it uses as a basis. Moreover, the implementation of Fictitious Play would be impractical on the high dimensional datasets we consider, due to its computational costs.

\textbf{Evaluating against strong adversarial attacks.} When evaluating a defense against adversarial examples, it is crucial to test the robustness of the method against the best possible attack. Accordingly, the defense method should be evaluated against attacks that were specifically tailored to it (a.k.a. adaptive attacks). In particular, when evaluating randomized algorithms, one should use Expectation over Transformation (EOT) to avoid gradient masking as pointed out by~\cite{athalye2018obfuscated} and~\cite{carlini2019evaluating}. More recently, \cite{tramer2020adaptive} emphasized that one should also make sure that EOT is computed properly\footnote{In order for the attack to succeed, it it more efficient to compute the expected transformation of the logits instead of taking the expectation over the loss. More details on this in the supplementary materials.}. Previous works such as~\cite{pruningDefenseICLR2018} and~\cite{Pinot2019} estimate the EOT through a Monte Carlo sampling which can introduce a bias in the attack if the sample size is too small. Since we assume perfect information for the Adversary, it knows the exact distribution of the mixture. Hence it can directly compute the expectation without using a sampling method, which avoid any bias. Table~\ref{tab_accuracy_under_attack} evaluates our method against strong adaptive attacks namely \textbf{Adaptive-}$\ell_{\infty}$\textbf{-PGD} and \textbf{Adaptive-}$\ell_{2}$\textbf{-C\&W}.

%\textbf{Symply applying EOT is not sufficient for $\ell_{2}$\textbf{-C\&W}}. We noticed that classical way of computing EOT\footnote{Evaluating the expectation over the logits.} was not optimal for all examples. For some of them, doing the expectation on the losses was more efficient for the Adversary. Hence, in order to evaluate our model against the best possible attack, we compute in parallel \textbf{Adaptive-}$\ell_{2}$\textbf{-C\&W} using EOT on the logits and \textbf{Adaptive-}$\ell_{2}$\textbf{-C\&W} using EOT on the losses on each example and take the perturbation that maximise the average error of the mixture. We refer the reader to the supplementary material for sanity-checks that show the two attacks are adaptive to a mixture of classifiers.

\textbf{Hard constraint parameter.} The typical value of $\epsilon$ in the hard constraint depends on the norm we consider in the problem setting. In this paper, we use an $\ell_2$ norm, however, the constraint parameter for $\ell_{\infty}$\textbf{-PGD} attack was initially set to be an $\ell_{\infty}$ constraint. In order to compare attacks of similar strength, we choose different threshold ($\epsilon_2$ or $\epsilon_{\infty}$) values which result in balls of equivalent volumes. For CIFAR10 an CIFAR100 datasets~\cite{krizhevsky2009learning}, which are $3\times32\times32$ dimensional spaces, this gives $ \epsilon_\infty = 0.03$ and $\epsilon_2 = 0.8$ (we also give results for $\epsilon_2$ equal to $0.6$ and $0.4$ as this values are sometimes used in the literature). Since \textbf{Adaptive-}$\ell_{2}$\textbf{-C\&W} attack creates an unbounded perturbation on the examples, we implemented the constraint from Equation~\ref{penalty::CW} by checking at test time whether the $\ell_2$-norm of the perturbation exceeds a certain threshold $\epsilon_2 \in \{0.4,0.6,0.8\}$. If it does, the adversarial example is disregarded, and we keep the natural example instead.

\textbf{Experimental results.} In Table~\ref{tab_accuracy_under_attack} we compare the accuracy, on CIFAR10 and CIFAR100, of our method and classical Adversarial Training under attack with \textbf{Adaptive-}$\ell_{\infty}$\textbf{-PGD} and \textbf{Adaptive-}$\ell_{2}$\textbf{-C\&W}, both run for $100$ iterations. We used $5$ times more iterations for the evaluation as we used during training, and carefully check for convergence. the rational behind this is that, for a classifier to be fully robust, its loss of accuracy should be controlled when the attacks are stronger than the ones it was trained on. For both attacks, both datasets and all thresholds (\emph{\emph{i.e.} the budget for a perturbation}), the accuracy under attack of our mixture is higher than the single classifier with Adversarial Training. Our defense is especially more robust than Adversarial Training when the threshold is high. 

\textbf{Extension to more than two classifiers}. In this paper we focus our experiments on a mixture of two classifiers to present a proof of concept of Theorem~\ref{theorem:MixedRegime}. Nevertheless, a mixture of more than two classifiers can be constructed by adding at each step $t$ a new classifier trained naturally on the dataset $\tilde{D}$ that contains adversarial examples against the mixture at step $t-1$. Since  $\tilde{D}$ has to be constructed from a mixture, one would have to use an adaptive attack as \textbf{Adaptive-}$\ell_{\infty}$\textbf{-PGD}. We refer the reader to the supplementary material for this extended version of the algorithm and for all the implementation details related to our experiments (architecture of models, optimization settings, hyper-parameters, etc.).

\section{Discussion \& Conclusion}

Finally, is there a classifier that ensures optimal robustness against all adversarial attacks? We gave a negative answer to this question in the deterministic regime, but part of the question remains open when considering randomized algorithms. We demonstrated that randomized defenses are more efficient than deterministic ones, and devised a simple method to implement them. 

\textbf{Game theoretical point of view.} There remains to study whether an Equilibrium exists in the Randomized regime. This question is appealing from a theoretical point of view, and requires to investigate the space of randomized Adversaries $\mathcal{P}((\mathcal{F}_{\mathcal{X}})^2)$. The characterization of this space is not straightforward, and would require strong results in the theory of optimal transport. A possible research direction is to quotient the space $(\mathcal{F}_{\mathcal{X}})^2$ so as to simplify the search in $\mathcal{P}((\mathcal{F}_{\mathcal{X}})^2)$ and the characterization of the Adversary's best responses. The study of this equilibrium is tightly related to that of the value of the game, which would be interesting for
obtaining min-max bounds on the accuracy under attack, as well as certificates of robustness for a set of classifiers.

\textbf{Advocating for more provable defenses.} Although the experimental results show that our mixture of classifiers outperforms Adversarial Training, our algorithm does not provide guarantees in terms of certified accuracy. As the literature on adversarial attacks and defenses demonstrated, better attacks always exist. This is why, more theoretical works need to be done to prove the robustness of a mixture created from this particular algorithm. More generally, our work advocates for the study of mixtures as a provable defense against adversarial attacks. 
One could, for example, build upon the connection between mixtures and noise injection to investigate a broader range of randomized strategies for the Defender, and devise certificates accordingly. 

\textbf{Improving Boosted Adversarial Training.} From an algorithmic point of view, BAT can be improved in several ways. For instance, the weights can be learned while choosing the new classifier for the mixture. This could lead to an improved accuracy under attack, but would lack some theoretical justifications that still need to be set up. 
Finally, tighter connections with standard boosting algorithms could be established to improve the analysis of BAT.

\section*{Acknowledgements}
We thank anonymous reviewers,  whose  comments  helped  us  improve  the  paper significantly. We also thank Rida Laraki and Guillaume Carlier for fruitful discussions on game theory as well as Alexandre Araujo for proof reading our experiments. This work was granted access to the HPC resources of IDRIS under the allocation 2020-101141 made by GENCI.

\bibliographystyle{icml2020}
\bibliography{biblio}

\newpage

\onecolumn

\icmltitle{Supplementary Material}

\setcounter{section}{0}
\setcounter{theorem}{0}
\setcounter{lemma}{0}

\section{Omitted proofs and Additional results}
\label{section::proofResults}

\paragraph{Notations.} Let us suppose that ($\mathcal{X}$,$\norm{.}$) is a normed vector space. $B_{\norm{.}}(x, \epsilon) = \left\{ z \in \mathcal{X} \mid \norm{x -z} \leq \epsilon \right\} $ is the closed ball of center $x$ and radius $\epsilon$ for the norm $\norm{.}$. Note that $\mathcal{H} := \{h:x \mapsto \sign g(x) \mid g:\mathcal{X}\rightarrow \mathbb{R} \textnormal{ continuous}\}$, with $\sign$ the function that outputs $1$ if $g(x)>0$, $-1$ if $g(x)<0$, and $0$ otherwise. Hence 
for any $(x,y)\sim D$, and $h \in \mathcal{H}$ one has $\mathds{1}\{ h(x) \neq y \} = \mathds{1}\{ g(x)y \leq 0 \}$. Finally, we denote $\nu_1$ and $\nu_{\text{-}1}$ respectively the probabilities of class $1$ and $\text{-}1$.

\paragraph{Introducing remarks.} Let us first note that in the paper, the penalties are defined with an $\ell_2$ norm. However, Lemma \ref{lemma:BRattMass} and~\ref{lemma:BRattCW} hold as long as $\mathcal{X}$ is an Hilbert space with dot product $< \mid >$ and associated norm $||.|| = \sqrt{< . \mid . >}$. We first demonstrate Lemma~\ref{lemma:BRattCW} with these general notations. Then we present the proof of Lemma~\ref{lemma:BRattMass} that follows the same schema. Note that, for Lemma~\ref{lemma:BRattMass}, we do not even need the norm to be Hilbertian, since the core argument rely on separation property of the norm, \emph{i.e.} on the property $\norm{x - y} = 0 \iff x=y$. \\

\setcounter{lemma}{1}
\begin{lemma}
\label{lemma:BRattCW}
Let $h \in \mathcal{H}$ and  $\boldsymbol{\phi} \in \BRD_{\penCW}(h)$. Then the following assertion holds: $$
    \phi_1 (x) = \left\{
    \begin{array}{ll}
    \pi(x) &\text{ if } x \in P_h(\epsilon_2)\\
    x &\text{ otherwise}.
    \end{array}
    \right.
$$
Where $\pi$ is the orthogonal projection on $(P_h)^{\complement}$. $\phi_{\text{-}1}$ is characterized symmetrically. 
\end{lemma}

\begin{proof}
%$\penCW$ is defined with an $\ell_2$ norm, but this result holds as long as $\mathcal{X}$ is an Hilbert space with dot product $< \mid >$ and associated norm $||.|| = \sqrt{< . \mid . >}$. We demonstrate the result with these general notations.

Let us first simplify the worst case adversarial risk for $h$. Recall that $h=\sign(g)$ with $g$ continuous. From the definition of adversarial risk we have:
\begin{align}
    & \sup\limits_{\boldsymbol{\phi} \in \left(\mathcal{F}_{\mathcal{X} \vert \epsilon_{2}}\right)^2} \scoreregCW(h,\boldsymbol{\phi}) \\
    = & \sup\limits_{\boldsymbol{\phi} \in \left(\mathcal{F}_{\mathcal{X} }\right)^2} \sum\limits_{y=\pm 1} \nu_y \expect_{X \sim \mu_y} \big[ \mathds{1}\left\{h\left(\phi_y(X)\right) \neq y \right\} - \lambda \norm{X-\phi_y(X)} - \infty \mathds{1}\left\{ \norm{X-\phi_y(X)} > \epsilon_2 \right\} \big]
    \\
    = & \sup\limits_{\boldsymbol{\phi} \in \left(\mathcal{F}_{\mathcal{X} }\right)^2} \sum\limits_{y=\pm 1} \nu_y \expect_{X \sim \mu_y} \big[ \mathds{1}\left\{g\left(\phi_y(X)\right)y \leq 0 \right\} - \lambda \norm{X-\phi_y(X)} - \infty \mathds{1}\left\{ \norm{X-\phi_y(X)} > \epsilon_2 \right\} \big]
    \\
    = &  \sum\limits_{y=\pm 1} \nu_y \sup\limits_{\phi_y \in \mathcal{F}_{\mathcal{X}}} \expect_{X \sim \mu_y} \big[ \mathds{1}\left\{g\left(\phi_y(X)\right)y \leq 0 \right\} - \lambda \norm{X-\phi_y(X)} - \infty \mathds{1}\left\{ \norm{X-\phi_y(X)} > \epsilon_2 \right\} \big]
\intertext{Finding $\phi_1$ and $\phi_1$ are two independent optimization problems, hence, we focus on characterizing $\phi_1$ (\emph{i.e.} $y=1$).}
    &\sup\limits_{\phi_1 \in \mathcal{F}_{\mathcal{X}}} \expect_{X \sim \mu_1} \big[ \mathds{1}\left\{g\left(\phi_1(X)\right) \leq 0 \right\} - \lambda \norm{X-\phi_1(X)} - \infty \mathds{1}\left\{ \norm{X-\phi_1(X)} > \epsilon_2 \right\} \big] \\
    = & \expect_{X \sim \mu_1} \left[ \essup\limits_{z \in B_{\norm{.}}(X,\epsilon_2)} \mathds{1}(g(z) \leq 0) - \lambda \norm{X-z} \right] 
    \\
    = & \int\limits_{\mathcal{X}} \essup\limits_{z \in B_{\norm{.}}(x,\epsilon_2)}  \mathds{1}\left\{g(z) \leq 0 \right \} - \lambda \norm{x-z}~~d\mu_1(x).
\intertext{
Let us now consider $(H_j)_{j \in J}$ a partition of $\mathcal{X}$, we can write.} 
&\sup\limits_{\phi_1 \in \mathcal{F}_{\mathcal{X}}} \expect_{X \sim \mu_1} \big[ \mathds{1}\left\{g\left(\phi_1(X)\right) \leq 0 \right\} - \lambda \norm{X-\phi_1(X)} - \infty \mathds{1}\left\{ \norm{X-\phi_1(X)} > \epsilon_2 \right\} \big] \\
    = & \sum\limits_{j \in J} \int\limits_{H_j} \essup\limits_{z \in B_{\norm{.}}(x,\epsilon_2)}  \mathds{1}\left\{g(z) \leq 0 \right \} - \lambda \norm{x-z}~~d\mu_1(x)
\end{align}

In particular, we consider here $H_0 =P_h^{\complement}$, $H_1=P_h \setminus P_h(\epsilon_2)$, and $H_2=P_h(\epsilon_2)$. 

\paragraph{For $x \in H_{0} = P_h^{\complement}$.} Taking $z=x$ we get $\mathds{1}\left\{g(z) \leq 0 \right \} - \lambda \norm{x-z} = 1$. Since for any $z \in \mathcal{X}$ we have $\mathds{1}\left\{g(z) \leq 0 \right \} - \lambda \norm{x-z} \leq 1$, this strategy is optimal. Furthermore, for any other optimal strategy $z'$, we would have $\norm{x-z'} = 0$, hence $z'=x$, and an optimal attack will never move the points of $H_{0}=P_h^{\complement}$.
   
\paragraph{For $x \in H_{1} = P_h \setminus P_h(\epsilon_2)$.} We have $B_{\norm{.}}(x,\epsilon_2) \subset P_h$ by definition of $P_h(\epsilon_2)$. Hence, for any $z \in B_{\norm{.}}(x,\epsilon_2)$, one gets $g(z) > 0$. Then $\mathds{1}\left\{g(z) \leq 0 \right \} - \lambda \norm{x-z} \leq 0$. The only optimal $z$ will thus be $z=x$, giving value $0$.

\paragraph{Let us now consider $x \in H_{2} = P_h(\epsilon_2)$ which is the interesting case where an attack is possible.} We know that $B_{\norm{.}}(x,\epsilon_2) \cap P_h^{\complement} \neq \emptyset$, and for any $z$ in this intersection, $\mathds{1}(g(z) \leq 0) = 1$. Hence :
    \begin{align}
        \essup\limits_{z \in B_{\norm{.}}(x,\epsilon_2)} \mathds{1}\left\{g(z) \leq 0 \right \} - \lambda \norm{x-z}
        = & \max(1 - \lambda \essinf_{z \in B_{\norm{.}}(x,\epsilon_2)\cap P_h^{\complement}} \norm{x-z}, 0)
        \\
        = & \max(1 - \lambda \pi_{ B_{\norm{.}}(x,\epsilon_2) \cap P_h^{\complement} }(x), 0)
    \end{align}
    Where $\pi_{ B_{\norm{.}}(x,\epsilon_2) \cap P_h^{\complement} }$ is the projection on the closure of $B_{\norm{.}}(x,\epsilon_2) \cap P_h^{\complement}$. Note that $\pi_{ B_{\norm{.}}(x,\epsilon_2) \cap P_h^{\complement} }$ exists: $g$ is continuous, so $B_{\norm{.}}(x,\epsilon_2) \cap P_h^{\complement}$ is a closed set, bounded, and thus compact, since we are in finite dimension. The projection is however not guaranteed to be unique since we have no evidence on the convexity of the set. Finally, let us remark that, since $\lambda \in (0,1)$, and $\epsilon_2 \leq 1$, one has $1 - \lambda \pi_{ B_{\norm{.}}(x,\epsilon_2) \cap P_h^{\complement} }(x) \geq 0$ for any $x \in H_2$.
Hence, on $P_h(\epsilon_2)$, the optimal attack projects all the points on the decision boundary. For simplicity, and since there is no ambiguity, we write the projection $\pi$.

\paragraph{Finally.} Since $H_0 \cup H_1 \cup H_2 =\mathcal{X}$, Lemma~\ref{lemma:BRattCW} holds. Furthermore, the score for this optimal attack is:
\begin{align}
&\sup\limits_{\phi \in \left(\mathcal{F}_{\mathcal{X} \vert \epsilon_{2}}\right)^2} \scoreregCW(h,\phi) \\
= & \sum\limits_{y=\pm 1} \nu_y \sum\limits_{j \in J} \int\limits_{H_j} \essup\limits_{z \in B_{\norm{.}}(x,\epsilon_2)}  \mathds{1}\left\{g(z)y \leq 0 \right \} - \lambda \norm{x-z} ~~d\mu_y(x)
\intertext{Since the value is $0$ on $P_h \setminus P_h(\epsilon_2)$ (resp. on $N_h \setminus N_h(\epsilon_2)$ ) for $\phi_1$ (resp. $\phi_{\text{-}1}$), one gets:}
= &  \nu_1 \left[ \int\limits_{P_h(\epsilon_2)} \big( 1-\lambda \norm{x- \pi(x)} \big) d\mu_1(x)  + \int\limits_{P_h^{\complement}} 1 d\mu_1(x) \right] + \nu_{\text{-}1}\left[ \int\limits_{N_h(\epsilon_2)} \big( 1-\lambda  \norm{x- \pi(x)} \big) d\mu_{\text{-}1}(x)  + \int\limits_{N_h^{\complement}} 1d\mu_{\text{-}1}(x)\right]
    \\
= &  \nu_1 \left[ \int\limits_{P_h(\epsilon_2)} \big( 1-\lambda \norm{x- \pi(x)} \big) d\mu_1(x) + \mu_1(P_h^{\complement}) \right] + \nu_{\text{-}1}\left[ \int\limits_{N_h(\epsilon_2)} \big( 1-\lambda  \norm{x- \pi(x)} \big) d\mu_{\text{-}1}(x) + \mu_{\text{-}1}(N_h^{\complement}) \right]
    \\
= & \Risk(h) + \nu_1 \int\limits_{P_h(\epsilon_2)} \big( 1-\lambda  \norm{x- \pi(x)} \big) d\mu_1(x) +\nu_{\text{-}1} \int\limits_{N_h(\epsilon_2)} \big( 1-\lambda \norm{x- \pi(x)} \big) d\mu_{\text{-}1}(x)
\end{align}

(16) holds since $\Risk(h) = \mathbb{P}(h(X) \neq Y)\mathbb{P}(g(X)Y \leq 0) =  \nu_1 \mu_1(P_h^{\complement}) + \nu_{\text{-}1} \mu_{\text{-}1}(N_h^{\complement})$. This provides an interesting decomposition of the adversarial risk into the risk without attack and the loss on the attack zone.

\end{proof}

\setcounter{lemma}{0}

\begin{lemma}
\label{lemma:BRattMass}
Let $h \in \mathcal{H}$ and $\boldsymbol{\phi} \in \BRD_{\penMass}(h)$. Then the following assertion holds: $$
    \left\{
    \begin{array}{ll}
    \phi_1(x) \in (P_h)^{\complement} &\text{ if }  x \in P_h(\epsilon_2) \\
    \phi_1(x) = x &\text{ otherwise}.
    \end{array}
    \right.
$$ 
Where $(P_h)^{\complement}$, the complement of $P_h$ in $\mathcal{X}$. $\phi_{\text{-}1}$ is characterized symmetrically.
\end{lemma}

\begin{proof} Following the same proof schema as before the adversarial risk writes as follows:  \begin{align}
    &\sup\limits_{\boldsymbol{\phi} \in \left(\mathcal{F}_{\mathcal{X} \vert \epsilon_{2}}\right)^2} \scoreregMass(h,\phi) \\
    = & \sup\limits_{\boldsymbol{\phi} \in \left(\mathcal{F}_{\mathcal{X} }\right)^2} \sum\limits_{y=\pm 1} \nu_y \expect_{X \sim \mu_y} \left[ \mathds{1}\left\{h\left(\phi_y(X)\right) \neq y \right\} - \lambda \mathds{1}\left\{X \neq \phi_y(X) \right\} - \infty \mathds{1}\left\{ \norm{X-\phi_y(X)} > \epsilon_2 \right\} \right]
    \\
    = & \sup\limits_{\boldsymbol{\phi} \in \left(\mathcal{F}_{\mathcal{X} }\right)^2} \sum\limits_{y=\pm 1} \nu_y \expect_{X \sim \mu_y} \left[ \mathds{1}\left\{g\left(\phi_y(X)\right)y \leq 0 \right\} - \lambda \mathds{1}\left\{X \neq \phi_y(X) \right\} - \infty \mathds{1}\left\{ \norm{X-\phi_y(X)} > \epsilon_2 \right\} \right]
    \\
    = &  \sum\limits_{y=\pm 1} \nu_y \sup\limits_{\phi_y \in \mathcal{F}_{\mathcal{X}}} \expect_{X \sim \mu_y} \left[ \mathds{1}\left\{g\left(\phi_y(X)\right)y \leq 0 \right\} - \lambda \mathds{1}\left\{X \neq \phi_y(X) \right\} - \infty \mathds{1}\left\{ \norm{X-\phi_y(X)} > \epsilon_2 \right\} \right]
\intertext{Finding $\phi_1$ and $\phi_1$ are two independent optimization problem, hence we focus on characterizing $\phi_1$ (\emph{i.e.} $y=1$).} 
    &\sup\limits_{\phi_1 \in \mathcal{F}_{\mathcal{X}}} \expect_{X \sim \mu_1} \big[ \mathds{1}\left\{g\left(\phi_1(X)\right) \leq 0 \right\} - \lambda \mathds{1}\left\{X \neq \phi_1(X) \right\} - \infty \mathds{1}\left\{ \norm{X-\phi_1(X)} > \epsilon_2 \right\} \big] \\
    = & \expect_{X \sim \mu_1} \left[ \essup\limits_{z \in B_{\norm{.}}(X,\epsilon_2)} \mathds{1}\left\{ g(z) \leq 0 \right\} - \lambda \mathds{1}\left\{X \neq z \right\} \right] 
    \\
    = & \int\limits_{\mathcal{X}} \essup\limits_{z \in B_{\norm{.}}(x,\epsilon_2)}  \mathds{1}\left\{g(z) \leq 0 \right \} - \lambda \mathds{1}\left\{x \neq z \right\}~~d\mu_1(x).
\intertext{Let us now consider $(H_j)_{j \in J}$ a partition of $\mathcal{X}$, we can write. }
&\sup\limits_{\phi_1 \in \mathcal{F}_{\mathcal{X}}} \expect_{X \sim \mu_1} \big[ \mathds{1}\left\{g\left(\phi_1(X)\right) \leq 0 \right\} - \lambda \mathds{1}\left\{X \neq \phi_1(X) \right\} - \infty \mathds{1}\left\{ \norm{X-\phi_1(X)} > \epsilon_2 \right\} \big] \\
    = & \sum\limits_{j \in J} \int\limits_{H_j} \essup\limits_{z \in B_{\norm{.}}(x,\epsilon_2)}  \mathds{1}\left\{g(z) \leq 0 \right \} - \lambda \mathds{1}\left\{x \neq z \right\}~~d\mu_1(x)
\end{align}

In particular, we can take $H_0 = P_h^{\complement}$, $H_1=P_h \setminus P_h(\epsilon_2)$, and $H_2=P_h(\epsilon_2)$. 

\paragraph{For $x\in H_0 = P_h^{\complement}$ or $x \in H_1 = P_h \setminus P_h(\epsilon_2)$.} With the same reasoning as before, any optimal attack will choose $\phi_1(x)=x$.

\paragraph{Let $x\in H_2 = P_h(\epsilon_2)$.} We know that $B_{\norm{.}}(x,\epsilon_2) \cap P_h^{\complement} \neq \emptyset$, and for any $z$ in this intersection, one has $g(z) \leq 0$ and $z \neq x$. Hence $ \essup\limits_{z \in B_{\norm{.}}(x,\epsilon_2)} \mathds{1}\left \{g(z) \leq 0 \right \} - \lambda \mathds{1}\left \{z\neq x\right \} = \max(1 - \lambda, 0)$. Since $\lambda \in (0,1)$ one has $\mathds{1}\left \{g(z) \leq 0 \right \} - \lambda \mathds{1}\left \{z\neq x\right \} = 1 -\lambda $ for any $z \in B_{\norm{.}}(x,\epsilon_2) \cap P_h^{\complement}$. Then any function that given a $x \in \mathcal{X}$ outputs $\phi_1(x) \in B_{\norm{.}}(x,\epsilon_2) \cap P_h^{\complement} $ is optimal on $H_2$.

\paragraph{Finally.} Since $H_0 \cup H_1 \cup H_2 =\mathcal{X}$, Lemma~\ref{lemma:BRattMass} holds.

\end{proof}

\setcounter{lemma}{2}

\begin{lemma}
\label{lemma:BRDefender}
Let us consider $\boldsymbol{\phi} \in \left(\mathcal{F}_{\mathcal{X} \vert \epsilon_{2}}\right)^2$. If we take $h \in \BRD(\boldsymbol{\phi})$, then for $y =1$ (resp. $y=\text{-}1$), and for any $B \subset P_h$ (resp. $B \subset N_h$) one has $$ \mathbb{P}(Y=y \vert X \in B) \geq \mathbb{P}(Y= -y \vert X \in B) $$ with $Y \sim \nu$ and for all $y \in \mathcal{Y}$, $X|(Y=y) \sim \phi_y \# \mu_y $.
\end{lemma}

\begin{proof}
We reason ad absurdum. Let us consider $y = 1$, the proof for $y=-1$ is symmetrical. Let us suppose that there exists $C \subset P_h$ such that $ \nu_{\text{-}1} \phi_{\text{-}1} \# \mu_{\text{-}1} (C) > \nu_{1}\phi_{1} \# \mu_{1}(C)$. We can then construct $h_1$ as follows: 
$$
    h_1 (x) = \left\{
    \begin{array}{ll}
    h(x) &\text{ if } x \notin C\\
    -1 &\text{ otherwise}.
    \end{array}
    \right.
$$

Since $h$ and $h_1$ are identical outside $C$, the difference between the adversarial risks of $h$ and $h_1$ writes as follows:
\begin{align}
    &\scoreregMass(h, \phi) - \scoreregMass(h_1,\phi) \\
    = & \sum\limits_{y = \pm 1} \nu_y \int\limits_{C} \big(\mathds{1}\left\{h(x) \neq y\right\} - \mathds{1}\left\{h_1(x) \neq y\right\} \big)~~d(\phi_y \# \mu_y)(x) \\
    = & \nu_{-1} \mathds{1}\left\{h(x) = 1\right\} \phi_{-1} \# \mu_{\text{-}1}(C) - \nu_{1} \mathds{1}\left\{h_1(x) \neq 1\right\} \phi_{1} \# \mu_{1} (C) \\
    = & \nu_{-1} \phi_{-1} \# \mu_{\text{-}1} (C) - \nu_{1} \phi_{1} \# \mu_{1}(C)
\end{align}
Since by hypothesis $\nu_{-1} \phi_{-1} \# \mu_{\text{-}1} (C) > \nu_{1} \phi_{1} \# \mu_{1}(C)$ the difference between the adversarial risks of $h$ and $h_1$ is strictly positive. This means that $h_1$ gives strictly better adversarial risk than the best response $h$. Since, by definition $h$ is supposed to be optimal, this leads to a contradiction. Hence Lemma~\ref{lemma:BRDefender} holds.
\end{proof}

\begin{addLemma}
\label{lemma:bonusBayesClassifier}
Let us assume that there is a probability measure $\zeta$ that dominates both $\phi_1 \# \mu_1$ and $\phi_{\text{-}1} \# \mu_{\text{-}1}$. 
Let us consider $\boldsymbol{\phi} \in \left(\mathcal{F}_{\mathcal{X} \vert \epsilon_{2}}\right)^2$. If we take $h \in \BRD(\boldsymbol{\phi})$, then $h$ is the Bayes Optimal Classifier for the distribution characterized by $(\nu, \phi_1 \# \mu_1, \phi_{\text{-}1} \# \mu_{\text{-}1})$.
\end{addLemma}

\begin{proof}
For simplicity, we denote $f_1 = \frac{(d\phi_{1} \# \mu_1)}{d\zeta}$ and $f_{-1} = \frac{d(\phi_{-1} \# \mu_{\text{-}1})}{d\zeta}$ the Radon-Nikodym derivatives of $\phi_{1} \# \mu_{1}$ and $\phi_{-1} \# \mu_{\text{-}1}$ w.r.t. $\zeta$. The best response $h$ minimizes adversarial risk under attack $\boldsymbol{\phi}$. This minimal risk writes:
\begin{align}
    &\inf\limits_{h\in \mathcal{H}}\scoreregMass(h,\phi) \\
    = & \inf\limits_{h\in \mathcal{H}} \sum\limits_{y=\pm 1} \nu_y \expect_{x \sim \mu_y} \left[ \mathds{1}\left\{h(\phi_y(x)) \neq y\right\} \right] - \lambda \pen\left(\boldsymbol{\phi}\right) .
    \intertext{Since the the penalty function does not depend on $h$, it suffices to seek $ \inf\limits_{h\in \mathcal{H}} \sum\limits_{y=\pm 1} \nu_y \int\limits_{\mathcal{X}} \mathds{1}\left\{h(x) \neq y\right \}~~ d(\phi_y \# \mu_y)(x)$. 
    Moreover thanks to the transfer theorem, one gets the following:}
    &\inf\limits_{h\in \mathcal{H}} \sum\limits_{y=\pm 1} \nu_y \int\limits_{\mathcal{X}}  \mathds{1}\left\{h(x) \neq y\right \} ~~d(\phi_y \# \mu_y)(x)\\
    = & \inf\limits_{h\in \mathcal{H}} \sum\limits_{y=\pm 1} \nu_y \int\limits_{\mathcal{X}} \mathds{1}\left \{h(x) \neq y \right\} f_y(x)~~d\zeta(x) \\
    = & \inf\limits_{h\in \mathcal{H}} \int\limits_{\mathcal{X}} \sum\limits_{y=\pm 1} \nu_y \mathds{1}\left \{h(x) \neq y \right \}f_y(x) ~~d\zeta(x). 
    \intertext{Finally, since the integral is bounded we get:}
    &\inf\limits_{h\in \mathcal{H}} \int\limits_{\mathcal{X}} \sum\limits_{y=\pm 1} \nu_y \mathds{1}\left \{h(x) \neq y \right \}f_y(x) ~~d\zeta(x) \\ = & \int\limits_{\mathcal{X}} \left[ \inf\limits_{h\in \mathcal{H}} \sum\limits_{y=\pm 1} \nu_y \mathds{1}\left\{h(x) \neq y\right\} f_y(x) \right]d\zeta(x).
\end{align}

 Hence, the best response $h$ is such that for every $x\in \mathcal{X}$, and $y \in \mathcal{Y}$, one has $h(x) = y$ if and only if $f_y(x) \leq f_{-y}(x)$. Thus, $h$ is the optimal Bayes classifier for the distribution $(\nu, \phi_1 \# \mu_1, \phi_{\text{-}1} \# \mu_{\text{-}1})$.
Furthermore, for $y =1$ (resp. $y=\text{-}1$), and for any $B \subset P_h$ (resp. $B \subset N_h$) one has:$$ \mathbb{P}(Y=y \vert X \in B) \geq \mathbb{P}(Y= -y \vert X \in B) $$ with $Y \sim \nu$ and for all $y \in \mathcal{Y}$, $X|(Y=y) \sim  \phi_y \# \mu_y$.

\end{proof}

\begin{theorem}[Non-existence of a pure Nash equilibrium]
\label{theorem::PureNashEquilibrium}
In our zero-sum game with $\lambda \in (0,1)$ and penalty $\Omega \in \{ \penMass,\penCW \}$, there is no Pure Nash Equilibrium.
\end{theorem}

\begin{proof}
Let $h$ be a classifier, $\phi \in \BRA(h)$ an optimal attack against $h$. We will show that $h \notin \BRD(\boldsymbol{\phi})$, i.e. that $h$ does not satisfy the condition from Lemma~\ref{lemma:BRDefender}. This suffices for Theorem~\ref{theorem::PureNashEquilibrium} to hold since it implies that there is no $(h, \boldsymbol{\phi}) \in \mathcal{H} \times \left(\mathcal{F}_{\mathcal{X} \vert \epsilon_{2}}\right)^2 $ such that
$h \in \BRD(\boldsymbol{\phi}) \text{ \emph{and} }\boldsymbol{\phi} \in \BRA(h).$

According to Lemmas~\ref{lemma:BRattMass} and~\ref{lemma:BRattCW}, whatever penalty we use, there exists $\delta >0$ such that $\phi_1 \# \mu_1\left(P_h(\delta)\right) = 0$ or $\phi_{-1} \# \mu_{\text{-}1}\left(N_h(\delta)\right) = 0$. Both cases are symmetrical, so let us assume that $P_h(\delta)$ is of null measure for the transported distribution conditioned by $y=1$. Furthermore we have $\phi_{-1} \# \mu_{\text{-}1}\left(P_h(\delta)\right) = \mu_{\text{-}1}\left(P_h(\delta)\right) > 0$ since $\phi_{-1}$ is the identity function on $P_h(\delta)$, and since $\mu_{\text{-}1}$ is of full support on $\mathcal{X}$. Hence we get the following:
\begin{align}
    & \phi_{-1} \# \mu_{\text{-}1}\left(P_h(\delta)\right) > \phi_1 \# \mu_{1}\left(P_h(\delta)\right). \intertext{Since the right side of the inequality is null, we also get:}
& \phi_{-1} \# \mu_{\text{-}1}\left(P_h(\delta)\right) \nu_{\text{-}1} > \phi_1 \# \mu_{1}\left(P_h(\delta)\right)  \nu_1.
\end{align}   
This inequality is incompatible with the characterization of best response for the Defender of Lemma~\ref{lemma:BRDefender}. Hence $h \notin \BRD(\boldsymbol{\phi})$.

\end{proof}

\begin{theorem}(Randomization matters)
\label{theorem:MixedRegime}
Let us consider $h_1 \in \mathcal{H}$, $\lambda \in (0,1)$, $\pen = \penMass$, $\boldsymbol{\phi} \in \BRD_{\pen}(h_1)$ and $h_2 \in \BRD(\boldsymbol{\phi})$. Then for any $\alpha \in (\max(\lambda, 1-\lambda),1)$ and for any $\boldsymbol{\phi}' \in \BRD_{\pen}(m^{\mathbf{q}}_{\mathbf{h}})$ one has
\begin{equation*}
     \scoreregMass(m^{\mathbf{q}}_{\mathbf{h}}, \boldsymbol{\phi}') < \scoreregMass(h_1, \boldsymbol{\phi}).
\end{equation*}
Where $\mathbf{h}=(h_1,h_2)$, $\mathbf{q}=(\alpha, 1-\alpha)$, and $m^{\mathbf{q}}_{\mathbf{h}}$ is the mixture of $\mathbf{h}$ by $\mathbf{q}$.
\end{theorem}

\begin{figure*}[!ht]
    \centering
    \includegraphics[width=0.5\textwidth]{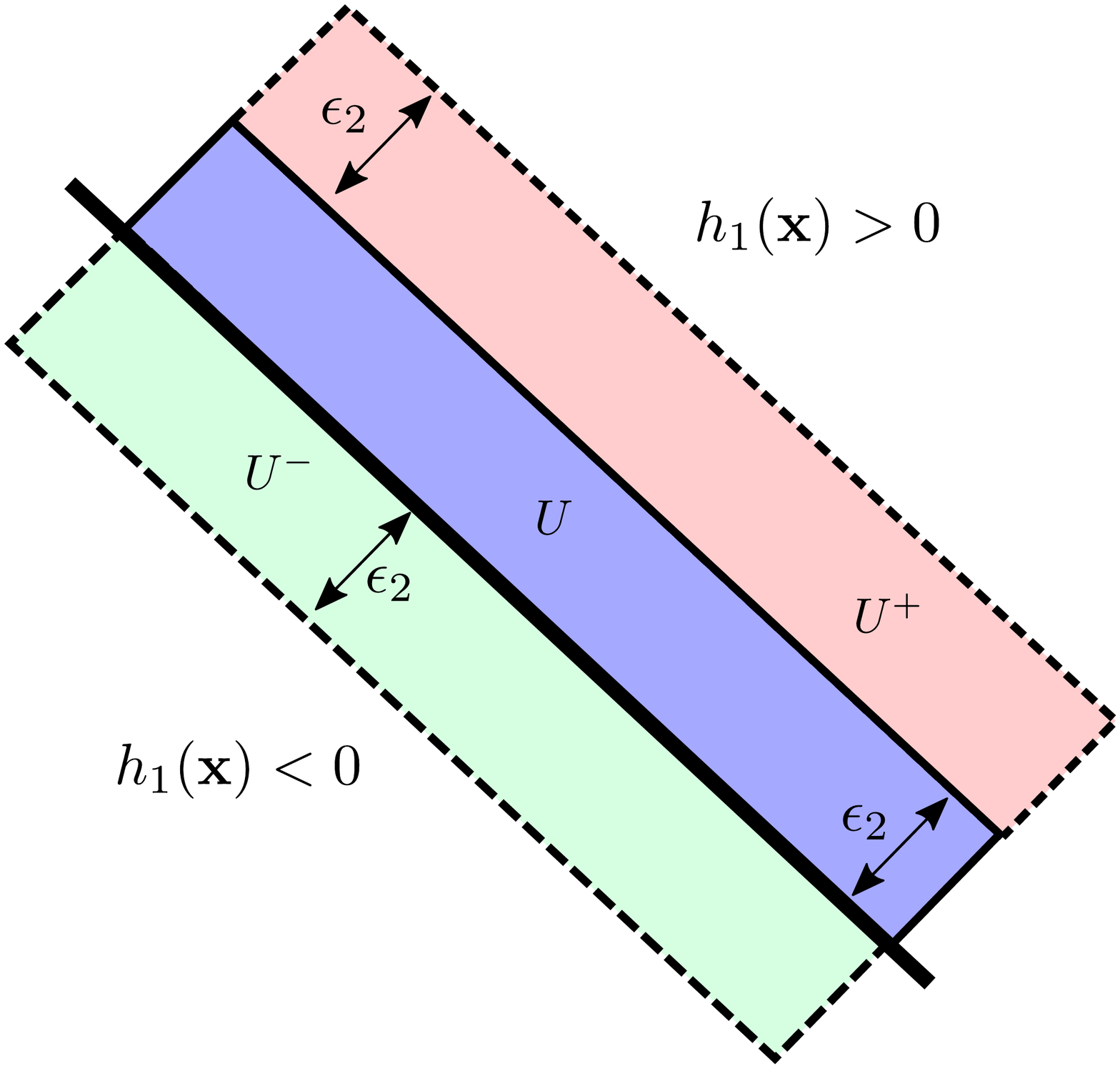}
    \caption{Illustration of the notations $U$, $U^+$, and $U^-$ for proof of Theorem~\ref{theorem:MixedRegime}.}
    \label{fig:preuve}
\end{figure*}

\begin{proof}
To demonstrate Theorem~\ref{theorem:MixedRegime}, let us denote $U = P_{h_1}(\epsilon_2)$ and define the \emph{$\epsilon_2$-dilation of $U$} as
$ U \oplus \epsilon_2 := \left\{ u + v \mid (u,v) \in U \times \mathcal{X} \text{ and } \norm{v}_p \leq \epsilon_2 \right\}.$ We can construct $h_2$ as follows
\begin{align*}
    h_2 (x) = \left\{
    \begin{array}{ll}
    -h_1(x) &\text{ if } x \in U\\
    h_1(x) &\text{ otherwise}.
    \end{array}
    \right.
\end{align*}
This means that $h_2$ changes the class of all points in $U$, and do not change the rest, compared to $h_1$. Then taking $\alpha \in \left(0,1 \right)$, we can define $m^{\boldsymbol{q}}_{\boldsymbol{h}}$, and $\boldsymbol{\phi}' \in \BRA(m^{\boldsymbol{q}}_{\boldsymbol{h}})$. We aim to find a condition on $\alpha$ so that the score of $m^{\boldsymbol{q}}_{\boldsymbol{h}}$ is lower than the score of $h_1$. Finally, let us recall that 
\begin{align*}
    &\scoreregMass(m^{\boldsymbol{q}}_{\boldsymbol{h}},\boldsymbol{\phi}') \\ =~& \nu_1  \int\limits_{\mathcal{X}} \essup\limits_{z \in B_{\norm{.}}(x,\epsilon_2)} \alpha \mathds{1} \left\{h_1(z) = \text{-}1 \right \} + (1- \alpha) \mathds{1}\left \{ h_2(z)  = \text{-}1  \right\} - \lambda \mathds{1} \left\{x \neq z\right\}~d\mu_1(x) \\
    +~& \nu_{\text{-}1}  \int\limits_{\mathcal{X}} \essup\limits_{z \in B_{\norm{.}}(x,\epsilon_2)} \alpha \mathds{1} \left\{h_1(z) = 1 0\right \} + (1- \alpha) \mathds{1}\left \{ h_2(z) = 1  \right\} - \lambda \mathds{1} \left\{x \neq z\right\}~d\mu_{\text{-}1}(x).
\end{align*}

The only terms that may vary between the score of $h_1$ and the score of $m^{\boldsymbol{q}}_{\boldsymbol{h}}$ are the integrals on $U$, $U \oplus \epsilon_2 \cap P_{h_1}$ and $\phi_{\text{-}1}^{-1}(U)$ -- inverse image of $U$ by $\phi_{\text{-}1}$. These sets represent respectively the points we mix on, the points that may become attacked -- when changing from $h_1$ to $m^{\boldsymbol{q}}_{\boldsymbol{h}}$ -- by moving them on $U$, and the ones that were -- for $h_1$ -- attacked before by moving them on $U$. Hence, for simplicity, we only write those terms. Furthermore, we denote $$ U^{+} := U \oplus \epsilon_2 \cap P_{h_1} \setminus U , \  U^{-} := \phi_{\text{-}1}^{-1}(U) \text{ and recall }
U := P_{h_1}(\epsilon_2). $$ One can refer to Figure~\ref{fig:preuve} for visual interpretation of this sets. We can now evaluate the worst case adversarial score for $h_1$ restricted to the above sets. Thanks to Lemma~\ref{lemma:BRattMass} that characterizes $\boldsymbol{\phi}$, we can write
\begin{align*}
&\scoreregMass(h_1,\boldsymbol{\phi})_{\mid U,\ U^+ ,\ U^-}\\ 
=~&
\left( 1 - \lambda \right) \times  \nu_1 \mu_1\left(U \right)
+ \nu_{\text{-}1} \mu_{\text{-}1}(U) \\[0.2cm]
+~& 0 \times \nu_1  \mu_1\left(U^+ \right) +  \nu_{\text{-}1} \mu_{\text{-}1} \left( U^+\right)
\\[0.2cm] 
+~&\nu_1\mu_1\left( U^- \right) + \left( 1 - \lambda \right) \times \nu_{\text{-}1}  \mu_{\text{-}1}\left( U^- \right).
\intertext{
Similarly, we can write the worst case adversarial score of the mixture on the sets we consider. Note that the max operator comes from the fact that the adversary has to make a choice between attacking the zone or just take advantage of the error due to randomization.
}
    &\scoreregMass(m^{\boldsymbol{q}}_{\boldsymbol{h}},\boldsymbol{\phi}')_{\mid U,\ U^+ ,\ U^-} \\
    =~ &
    \max \left( 1-\alpha, 1-\lambda \right) \times \nu_1 \mu_1\left(U\right) + \max \left( \alpha, 1-\lambda \right) \times \nu_{\text{-}1} \mu_{\text{-}1}(U)
    \\[0.2cm]
    +~&\max \left( 0, 1-\alpha-\lambda \right) \times \nu_1
     \mu_1\left(U^+ \right) + \nu_{\text{-}1} \mu_{\text{-}1} \left( U^+ \right) \\[0.2cm]
    +~&\nu_1\mu_1 \left( U^- \right) + \max \left(0, \alpha - \lambda \right)  \times \nu_{\text{-}1} \mu_{\text{-}1}\left(U^-\right).
\intertext{
Computing the difference between these two terms, we get the following
}
&  \scoreregMass(h_1,\boldsymbol{\phi}) - \scoreregMass(m^{\boldsymbol{q}}_{\boldsymbol{h}},\boldsymbol{\phi}') \numberthis \label{eq::difference1bis} \\
    =~& \left( 1- \lambda - \max \left( 1-\alpha, 1 -\lambda \right)  \right) \times \nu_1 \mu_1 \left( U \right) \numberthis \label{eq::difference2bis} 
    \\[0.2cm] 
    +~&  \left( 1 - \max \left( \alpha, 1-\lambda \right) \right) \times \nu_{\text{-}1} \mu_{\text{-}1} \left(U \right)  \numberthis \label{eq::difference3bis}
    \\[0.2cm]
    -~& \max \left(0, 1-\alpha- \lambda\right) \times \nu_1 \mu_1\left(U^+ \right)  \numberthis \label{eq::difference4bis}
     \\[0.2cm] + &  \left( 1 - \lambda - \max \left(0, \alpha - \lambda \right)\right) \times \nu_{\text{-}1} \mu_{\text{-}1} \left(U^-\right) \numberthis \label{eq::difference5bis}
\end{align*}
    
Let us now simplify Equation~(\ref{eq::difference1bis}) using additional assumptions.
\begin{itemize}
     \item First, we have that Equation~\eqref{eq::difference3bis} is equal to $$\min \left( 1 - \alpha, \lambda \right) \mu_{\text{-} 1}(U) \nu_{\text{-}1} > 0. $$
     Thus, a sufficient condition for the difference between the adversarial scores to be positive is to have the other terms greater or equal to $0$.
    \item To have Equation~\eqref{eq::difference2bis} $\geq 0$ we can always set $\max \left( 1-\alpha, 1-\lambda \right) = 1-\lambda.$
    This gives us $\alpha \geq  \lambda$.
    \item Also note that to get \eqref{eq::difference4bis} $\geq 0$, we can force $\max \left( 1-\alpha- \lambda, 0 \right)= 0.$ This gives us $\alpha \geq 1- \lambda$.
 \item Finally, since $\alpha \geq \lambda$, we have that
$1 - \lambda - \max \left(0, \alpha - \lambda \right) = 1-\alpha$ thus Equations~(\ref{eq::difference5bis}) $> 0$.
\end{itemize}
With the above simplifications, we have $\eqref{eq::difference1bis} > 0$ for any $\alpha > \max(\lambda, 1- \lambda)$ which concludes the proof.
\end{proof}

\begin{theorem}(Randomization matters)
\label{theorem:MixedRegimebis}
Let us consider $h_1 \in \mathcal{H}$, $\lambda \in (0,1)$, $\pen = \penCW$, $\boldsymbol{\phi} \in \BRD_{\pen}(h_1)$ and $h_2 \in \BRD(\boldsymbol{\phi})$. Let us take $\delta \in (0,\epsilon_2)$, then for any $\alpha \in (\max( 1- \lambda \delta , \lambda (\epsilon_2 - \delta)),1)$ and for any $\boldsymbol{\phi}' \in \BRD_{\pen}(m^{\mathbf{q}}_{\mathbf{h}})$ one has
\begin{equation*}
     \scoreregCW(m^{\mathbf{q}}_{\mathbf{h}}, \boldsymbol{\phi}') < \scoreregCW(h_1, \boldsymbol{\phi}).
\end{equation*}
Where $\mathbf{h}=(h_1,h_2)$, $\mathbf{q}=(\alpha, 1-\alpha)$, and $m^{\mathbf{q}}_{\mathbf{h}}$ is the mixture of $\mathbf{h}$ by $\mathbf{q}$.
\end{theorem}

\begin{figure*}[!ht]
    \centering
    \includegraphics[width=0.5\textwidth]{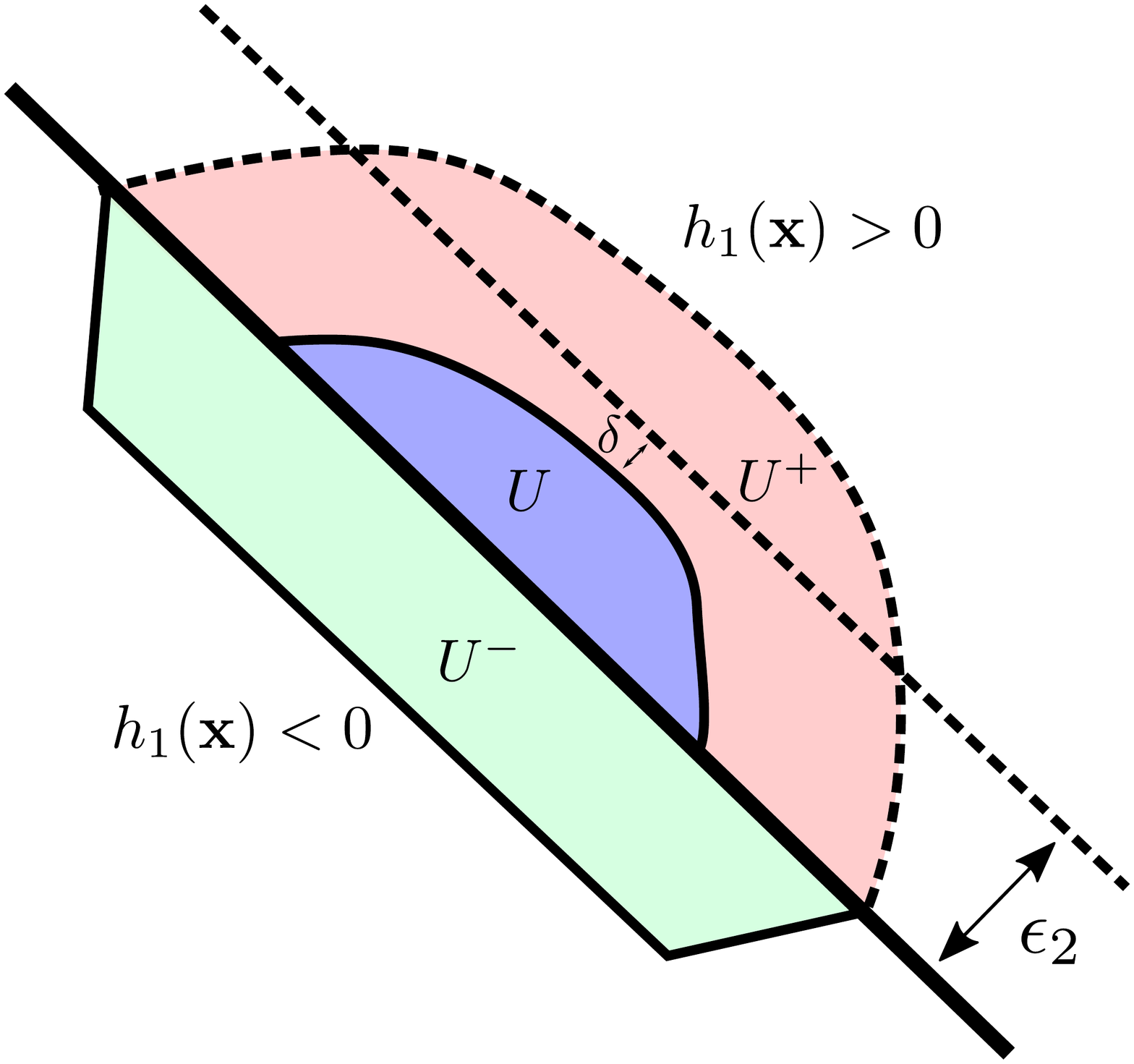}
    \caption{Illustration of the notations $U$, $U^+$, $U^-$ and $\delta$ for proof of Theorem~\ref{theorem:MixedRegimebis}.}
    \label{fig:preuve2}
\end{figure*}

\begin{proof}
Let us take $U \subset P_{h_1}(\epsilon_2)$ such that $$\min\limits_{x \in U}\Vert x - \proj_{P_h \setminus P_h(\epsilon_2)}( x) \Vert = \delta \in (0, \epsilon_2) $$. We construct $h_2$ as follows.
\begin{align*}
    h_2 (x) = \left\{
    \begin{array}{ll}
    -h_1(x) &\text{ if } x \in U\\
    h_1(x) &\text{ otherwise}.
    \end{array}
    \right.
\end{align*}
This means that $h_2$ changes the class of all points in $U$, and do not change the rest. Let $\alpha \in \left(0,1 \right)$, the corresponding mixture $m^{\boldsymbol{q}}_{\boldsymbol{h}}$, and $\boldsymbol{\phi}' \in \BRA(m^{\boldsymbol{q}}_{\boldsymbol{h}})$. We will find a condition on $\alpha$ so that the score of $m^{\boldsymbol{q}}_{\boldsymbol{h}}$ is lower than the score of $h_1$. Recall that
\begin{align*}
    &\scoreregCW(m^{\boldsymbol{q}}_{\boldsymbol{h}},\boldsymbol{\phi}') \\ =~& \nu_1  \int\limits_{\mathcal{X}} \essup\limits_{z \in B_{\norm{.}}(x,\epsilon_2)} \alpha \mathds{1} \left\{h_1(z) = \text{-}1  \right \} + (1- \alpha) \mathds{1}\left \{ h_2(z) = \text{-}1  \right\} - \lambda \norm{x - z} ~d\mu_1(x) \\
    +~& \nu_{\text{-}1}  \int\limits_{\mathcal{X}} \essup\limits_{z \in B_{\norm{.}}(x,\epsilon_2)} \alpha \mathds{1} \left\{h_1(z) = 1\right \} + (1- \alpha) \mathds{1}\left \{ h_2(z) = 1  \right\} - \lambda \norm{x - z}~d\mu_{\text{-}1}(x) .
\end{align*}
As we discussed in proof of Theorem~\ref{theorem:MixedRegime}, the only terms that may vary between the score of $h_1$ and the score of $m^{\boldsymbol{q}}_{\boldsymbol{h}}$ are the integrals on $U$, $U \oplus \epsilon_2 \cap P_{h_1}$ and $\phi_{\text{-}1}^{-1}(U)$. Hence, for simplicity, we only write those terms. Furthermore, we denote $$ U^{+}  := U \oplus \epsilon_2 \cap P_{h_1} \setminus U , \  U^{-}  := \phi_{\text{-}1}^{-1}(U) \text{ and }
P_{\epsilon_2}  := P_{h_1}(\epsilon_2). $$ 
One can refer to Figure~\ref{fig:preuve2} for a visual interpretation of this ensembles.
We can now evaluate the worst case adversarial score for $h_1$ restricted to the above sets. Thanks to Lemma~\ref{lemma:BRattCW} that characterizes $\boldsymbol{\phi}$, we can write
\begin{align*}
& \scoreregCW(h_1,\boldsymbol{\phi})\\ =~&
 \nu_1 \int\limits_{U} \left( 1 - \lambda \Vert x - \proj_{P_{h_1}^{\complement}}(x) \Vert \right) d\mu_1(x)
+ \nu_{\text{-}1} \mu_{\text{-}1}(U)
\\ +~& \nu_1 \int\limits_{ U^+ \setminus P_{\epsilon_2}} 0 \text{ }d\mu_1(x) + \nu_{\text{-}1} \mu_{\text{-}1}\left(U^+ \setminus P_{\epsilon_2} \right)
\\ +~&  \nu_1 \int\limits_{U^+ \cap P_{\epsilon_2}} \left( 1 - \lambda \Vert x - \proj_{P_{h_1}^{\complement}}(x) \Vert \right)d\mu_1(x) +  \nu_{\text{-}1} \mu_{\text{-}1}\left(U^+ \cap P_{\epsilon_2} \right) 
\\ +~& \nu_1 \mu_1\left(U^-\right) + \nu_{\text{-}1} \int\limits_{U^-}\Big( 1 - \lambda \Vert x - \proj_U(x) \Vert \Big) d\mu_{\text{-}1}(x).
\intertext{
Similarly we can evaluate the worst case adversarial score for the mixture,}
   &\scoreregCW(m^{\boldsymbol{q}}_{\boldsymbol{h}},\boldsymbol{\phi}') \\
   =~&
    \nu_1 \int\limits_U \max \left( 1-\alpha, 1-\lambda \Vert x - \proj_{P_{h_1}^{\complement}}(x) \Vert \right)~~d\mu_1(x)
    \\ +~& \nu_{\text{-}1} \int\limits_{U} \max \left( \alpha, 1-\lambda \Vert x - \proj_{U^+}(x)\Vert \right)~~d\mu_{\text{-}1}(x)
    \\ +~& \nu_1 \int\limits_{U^+ \setminus P_{\epsilon_2} }
    \max \left( 0, 1-\alpha-\lambda \Vert x - \proj_U(x) \Vert \right) ~~d\mu_1(x) + \nu_{\text{-}1} \mu_{\text{-}1} \left( U^+ \setminus P_{\epsilon_2} \right)
    \\ +~& \nu_1 \int\limits_{ U^+ \cap P_{\epsilon_2}} \max \left( 1-\alpha- \lambda \Vert x - \proj_U(x) \Vert, 1 - \lambda \Vert x - \proj_{P_{h_1}^{\complement}}(x) \Vert \right)~~d\mu_1(x)
    \\ +~& \nu_{\text{-}1} \mu_{\text{-}1} \left(U^+ \cap P_{\epsilon_2}\right) + \nu_1 \mu_1\left(U^-\right) 
    \\[0.2cm] +~& \nu_{\text{-}1}\int\limits_{U^-} \max \left(0, 1 - \lambda \Vert x- \proj_{N_{h_1}^{\complement} \setminus U }(x)\Vert, \alpha - \lambda \Vert x-\proj_U(x) \Vert \right) d\mu_{\text{-}1}(x).
\end{align*}

Note that we need to take into account the special case of the points in the dilation that were already in the attacked zone before, and that can now be attacked in two ways, either by projecting on $U$ -- but that works with probability $\alpha$, since the classification on $U$ is now randomized -- or by projecting on $P_{h_1}^{\complement}$, which works with probability 1 but may use more distance and so pay more penalty. We can now compute the difference between both scores.
\begin{align*}
     &  \scoreregCW(h_1,\boldsymbol{\phi}) - \scoreregCW(m^{\boldsymbol{q}}_{\boldsymbol{h}},\boldsymbol{\phi}') \numberthis \label{eq::difference1}\\
   =~&\nu_1 \int\limits_U  1-\lambda \Vert x - \proj_{P_{h_1}^{\complement}}(x) \Vert - \max \left( 1-\alpha, 1-\lambda \Vert x - \proj_{P_{h_1}^{\complement}}(x) \Vert \right)  d\mu_1(x) \numberthis \label{eq::difference2}
    \\ +~& \nu_{\text{-}1} \int\limits_{U} 1 - \max \left( \alpha, 1-\lambda \Vert x - \proj_{U^+}(x)\Vert \right) d\mu_{\text{-}1}(x) \numberthis \label{eq::difference3}
    \\-~& \nu_1 \int\limits_{U^+ \setminus P_{\epsilon_2}} \max \left( 1-\alpha- \lambda \Vert x - \proj_U(x) \Vert, 0 \right) d\mu_1(x) \numberthis \label{eq::difference6} 
    \\
    +~& \nu_1 \int\limits_{U^+ \cap P_{\epsilon_2}} 1-\lambda \Vert x - \proj_{P_{h_1}^{\complement}}(x) \Vert \\ 
    -~&  \max \left( 1-\alpha- \lambda \Vert x - \proj_U(x) \Vert, 1 - \lambda \Vert x - \proj_{P_{h_1}^{\complement}}(x) \Vert \right)  d\mu_1(x) \numberthis\label{eq::difference4two}
    \\ +~& \nu_{\text{-}1} \int\limits_{U^-} 1 - \lambda \Vert x - \proj_{U}(x) \Vert  \\
    -~&  \max \left(0, 1 - \lambda \Vert x- \proj_{N_{h_1}^{\complement} \setminus U}(x)\Vert, \alpha - \lambda \Vert x-\proj_U(x) \Vert \right) d\mu_{\text{-}1}(x). \numberthis \label{eq::difference5two}
    \end{align*}
Let us simplify Equation~(\ref{eq::difference1}) using using additional hypothesis:
\begin{itemize}
    \item First, note that Equation~\eqref{eq::difference3}> 0. Then a sufficient condition for the difference to be strictly positive is to ensure that other lines are $\geq 0$. 
    \item In particular to have $\eqref{eq::difference2} \geq 0$ it is sufficient to have for all $x \in U$ $$ \max \left( 1-\alpha, 1-\lambda \Vert x - \proj_{P_{h_1}^{\complement}}(x) \Vert \right) = 1-\lambda \Vert x - \proj_{P_{h_1}^{\complement}}(x) \Vert.$$
    
    This gives us $\alpha \geq \lambda (\epsilon_2-\delta) \geq  \lambda \max\limits_{x \in U}\Vert x - \proj_{P_{h_1}^{\complement}}(x) \Vert.$
    \item Similarly, to have $\eqref{eq::difference6} \geq 0$, we should set for all $x \in U^+ \setminus P_{\epsilon_2}$ $$ \alpha \geq 1 - \lambda \Vert x - \proj_U(x) \Vert.$$
    Since $\min\limits_{x \in U^+ \setminus P_{\epsilon_2}}\Vert x - \proj_U(x) \Vert =\delta$, we get the condition $\alpha \geq 1 -\lambda \delta$.
    \item Finally \eqref{eq::difference5two} $\geq 0$, since by definition of $U^-$, for any $x \in U^-$ we have $$\Vert x -\proj_{N_{h_1}^{\complement}\setminus U}(x) \Vert \geq \Vert x - \proj_{U}(x) \Vert.$$ 
\end{itemize}
Finally, by summing all these simplifications, we have $\eqref{eq::difference1} > 0$. Hence the result hold for any $\alpha > \max(1 - \lambda \delta, \lambda (\epsilon_2 - \delta) )$
\end{proof}
\section{Experimental results}

In the experimental section, we consider $\mathcal{X}=[0,1]^{3\times 32 \times 32}$ to be the set of images, and $\mathcal{Y}=\{1,...,10\}$ or $\mathcal{Y}=\{1,...,100\}$ according to the dataset at hand. 

\subsection{Adversarial attacks} 

Let $(x,y) \sim D$ and $h \in \mathcal{H}$. We consider the following attacks:

\textbf{(i) $\ell_\infty$\textbf{-PGD} attack.} In this scenario, the Adversary maximizes the loss objective function, under the constraint that the $\ell_\infty$ norm of the perturbation remains bounded by some value $\epsilon_\infty$. To do so, it recursively computes:
\begin{equation}
    x^{t+1}=\Pi_{B_{\norm{.}}(x,\epsilon_\infty)}\left[x^t
    +\beta \sign \left(\nabla_x\mathcal{L}\left(h\left(x^t\right),y\right)\right)\right]
    \label{eqn::projectionPGD}
\end{equation}
where $\mathcal{L}$ is some differentiable loss (such as the cross-entropy), $\beta$ is a gradient step size, and $\Pi_S$ is the projection operator on $S$. One can refer to~\cite{madry2018towards} for implementation details.

\textbf{(ii) $\ell_{2}$\textbf{-C\&W} attack.} In this attack, the Adversary optimizes the following objective:
  \begin{equation}
    \argmin_{\tau \in \mathcal{X}} \norm{\tau}_2+ \lambda \times \text{cost}(x+\tau)
    \label{eqn::CWproblem}
\end{equation}
where $\text{cost}(x+\tau)<0$ if and only if $h(x+\tau) \neq y$. 
The authors use a change of variable $\tau=\frac{1}{2}(\tanh(w)-x +1)$ to ensure that $x+\tau \in \mathcal{X}$, a binary search to optimize the constant $\lambda$, and Adam or SGD to compute an approximated solution. One should refer to~\cite{carlini2017towards} for implementation details.

\subsection{Experimental setup}

\label{section:exp_setup}

\paragraph{Datasets.}To illustrate our theoretical results we did experiments on the \textbf{CIFAR10} and \textbf{CIFAR100} datasets. See~\cite{cifar10} for more details. 

\paragraph{Classifiers.} All the classifiers we use are WideResNets (see~\cite{ZagoruykoK16}) with 28 layers, a widen factor of 10, a dropout factor of 0.3 and LeakyRelu activations with a 0.1 slope.

\paragraph{Natural Training.} To train an undefended classifier we use the following hyperparameters. 
        \begin{itemize}
            \item \textbf{Number of Epochs:} 200
            \item \textbf{Batch size:} 128
            \item \textbf{Loss function:} Cross Entropy Loss
            \item \textbf{Optimizer : } SGD algorithm with momentum 0.9, weight decay of $2\times10^{-4}$ and a learning rate that decreases during the training as follows: 
            \begin{align}
                lr = \left\{
                        \begin{matrix}
                        &0.1 & \text{if} & 0 &\leq & \text{epoch} & <& 60\\
                        &0.02 & \text{if} & 60 &\leq & \text{epoch} & <& 120\\
                        &0.004 & \text{if} & 120 &\leq & \text{epoch} & <& 160\\
                        &0.0008 & \text{if} & 160 &\leq & \text{epoch} & <& 200\\
                        \end{matrix} \notag
                        \right.
            \end{align}
        \end{itemize}
\paragraph{Adversarial Training.}To adversarially train a classifier we use the same hyperparameters as above, and generate adversarial examples using the $\ell_\infty$\textbf{-PGD} attack with 20 iterations. When considering that the input space is $[0,255]^{3 \times 32 \times 32}$, on \textbf{CIFAR10} and \textbf{CIFAR100}, a perturbation is considered to be imperceptible for $\epsilon_\infty=8$. Here, we consider $\mathcal{X}=[0,1]^{3 \times 32 \times 32}$ which is the normalization of the pixel space $[0.255]^{3 \times 32 \times 32}$. Hence, we choose  $\epsilon_2=0.031$ ($\approx 8/255$) for each attack. Moreover, the step size we use for $\ell_\infty$\textbf{-PGD} is $0.008$ ($\approx 2/255$), we use a random initialization for the gradient descent and we repeat the procedure three times to take the best perturbation over all the iterations \emph{i.e} the one that maximises the loss.
For the $\ell_\infty$\textbf{-PGD} attack against the mixture $m^{\mathbf{q}}_{\mathbf{h}}$, we use the same parameters as above, but compute the gradient over the loss of the expected logits (as explained in the main paper).

\paragraph{Evaluation Under Attack.}At evaluation time, we use 100 iterations instead of 20 for \textbf{Adaptive-}$\ell_\infty$\textbf{-PGD}, and the same remaining hyperparameters as before. For the \textbf{Adaptive-}$\ell_{2}$\textbf{-C\&W} attack, we use 100 iterations, a learning rate equal to 0.01, 9 binary search steps, and an initial constant of 0.001. We give results for several different values of the rejection threshold: $\epsilon_2 \in \{0.4,0.6,0.8\}$. 

\paragraph{Computing \textbf{Adaptive-}$\ell_2\textbf{-C\&W}$ on a mixture} To attack a randomized model, it is advised in the literature \cite{tramer2020adaptive} to compute the expected logits returned by this model. However this advice holds for randomized models that return logits in the same range for a same example (\emph{e.g.} classifier with noise injection). Our randomized model is a mixture and returns logits that depend on selected classifier. Hence, for a same example, the logits can be very different. This phenomenon made us notice that for some example in the dataset, computing the expected loss over the classifier (instead of the expected logits) performs better to find a good perturbation (it can be seen as computing the expectation of the logits normalized thanks to the loss). To ensure a fair evaluation of our model, in addition of using EOT with the expected logits, we compute in parallel EOT with the expected loss and take the perturbation that maximizes the expected error of the mixture. See the submitted code for more details.

\paragraph{Library used.} We used the Pytorch and Advertorch libraries for all implementations.

\paragraph{Machine used.} 6 Tesla V100-SXM2-32GB GPUs

\subsection{Experimental details}

\paragraph{Sanity checks for Adaptive attacks}

In \cite{tramer2020adaptive}, the authors give a lot of sanity checks and good practices to design an Adaptive attacks. We follow them and here are the information for \textbf{Adaptive-}$\ell_{\infty}$\textbf{-PGD} : 

\begin{itemize}
    \item We compute the gradient of the loss by doing the expected logits over the mixture.
    \item The attack is repeated 3 times with random start and we take the best perturbation over all the iterations.
    \item When adding a constant to the logits, it doesn't change anything to the attack 
    \item When doing 200 iterations instead of 100 iterations, it doesn't change the performance of the attack
    \item When increasing the budget $\epsilon_\infty$, the accuracy goes to 0, which ensures that there is no gradient masking. Here are some values to back this statement: 
    \begin{table}[H]
        \centering
        \begin{tabular}{c|c|c|c|c}
                 Epsilon & 0.015 & 0.031 & 0.125 & 0.250 \\
    Accuracy & 0.638 & 0.546 & 0.027 & 0.000
        \end{tabular}
        \caption{Evolution of the accuracy under \textbf{Adaptive-}$\ell_{\infty}$\textbf{-PGD} attack depending on the budget $\epsilon_\infty$}
        \label{tab:my_label}
    \end{table}
    \item The loss doesn't fluctuate at the end of the optimization process.
    
\end{itemize}

\paragraph{Selecting the first element of the mixture.}Our algorithm creates classifiers in a boosting fashion, starting with an adversarially trained classifier. There are several ways of selecting this first element of the mixture: use the classifier with the best accuracy under attack (option 1, called bestAUA), or rather the one with the best natural accuracy (option 2). Table~\ref{tab_supp_cifar10} compares both options. 

Beside the fact that any of the two mixtures outperforms the first classifier, we see that the fisrt option always outperforms the second. In fact, when taking option 1 (bestAUA = True) the accuracy under $\ell_\infty$\textbf{-PGD} attack of the mixture is $3\%$ better than with option 2 (bestAUA = False). One can also note that both mixtures have the same natural accuracy ($0.80$), which makes the choice of option 1 natural.

\begin{table}[H]
\begin{center}
\begin{tabular}{ c c c c c} 
\toprule
 \textbf{Training method} & NA of the $1^{st}$ clf & AUA of the $1^{st}$ clf & NA of the mixture & AUA of the mixture \\ 
 \toprule
 BAT (bestAUA=True) & 0.77 & \textbf{0.46} &  \textbf{0.80} & \textbf{0.55} \\  
 BAT (bestAUA=False) & \textbf{0.83} & 0.42 & \textbf{0.80} & 0.52 \\ 
\bottomrule
\end{tabular}
\end{center}
\caption{Comparison of the mixture that has as first classifier the best one in term of natural accuracy and the mixture that has as first classifier the best one in term of Accuracy under attack. The accuracy under attack is computed with the $\ell_\infty$\textbf{-PGD} attack. NA means matural accuracy, and AUA means accuracy under attack.}
\label{tab_supp_cifar10}
\end{table}

\subsection{Extension to more than two classifiers}

As we mention in the main part of the paper, a mixture of more than two classifiers can be constructed by adding at each step $t$ a new classifier trained naturally on the dataset $\tilde{D}$ that contains adversarial examples against the mixture at step $t-1$. Since  $\tilde{D}$ has to be constructed from a mixture, one would have to use an adaptive attack as \textbf{Adaptive-}$\ell_{\infty}$\textbf{-PGD}. Here is the algorithm for the extented version : 

\begin{algorithm}[H]
\SetAlgoLined
 \textbf{Input} : $n$ the number of classifiers, $D$ the training data set and $\alpha$ the weight update parameter. \\
 
 Create and adversarially train $h_1$ on $D$ \\
 $\mathbf{h} = (h_1)$ ; $\mathbf{q} = (1)$ \\
\For{$i = 2, \dots, n$}{
    Generate the adversarial data set $\tilde{D}$ against $m^{\mathbf{q}}_{\mathbf{h}}$. \\
    Create and naturally train $h_i$ on $\tilde{D}$ \\
 
    $q_k \leftarrow (1 - \alpha)q_k \ \ \ \  \forall k \in  [ i-1 ]$ \\
    $q_i \leftarrow \alpha$\\

    ${\bf q} \leftarrow \left(\alpha, \ldots, q_i\right)$ \\
    ${\bf h} \leftarrow (h_1,\ldots, h_i)$

}
return $m^{\mathbf{q}}_{\mathbf{h}}$ \\

    \caption{Boosted Adversarial Training}
 \label{algorithm:Boosting}
\end{algorithm}

Here to find the parameter $\alpha$, the grid search is more costly. In fact in the two-classifier version we only need to train the first and second classifier without taking care of $\alpha$, and then test all the values of $\alpha$ using the same two classifier we trained. For the extended version, the third classifier (and all the other ones added after) depends on the first classifier, the second one and their weights $1 - \alpha$ and $\alpha$. Hence the third classifier for a certain value of $\alpha$ can't be use for another one and, to conduct the grid search, one have to retrain all the classifiers from the third one.
Naturally the parameters $\alpha$ depends on the number of classifiers $n$ in the mixtures.

\end{document}